\documentclass[envcountsame]{svproc}

\usepackage{url}

\usepackage[utf8]{inputenc}
\usepackage[T1]{fontenc}
\usepackage[main=english]{babel}%
\usepackage{nicefrac}       
\usepackage{amssymb}        
\usepackage{graphicx}		
\usepackage{enumitem} 		
\usepackage{epstopdf} 		
\usepackage{stmaryrd}
\usepackage[title]{appendix} 
\usepackage[export]{adjustbox} 
\graphicspath{{Figures/}}
\usepackage[linecolor=blue!60!,backgroundcolor=blue!10!,textwidth=2.5cm]{todonotes}

\usepackage{hyperref}
\usepackage{nicefrac}
\usepackage{mathtools}

\input{macrogilles_main_modif_sp}

\def\Ff{{\mathcal F}}

\def\RR{{\mathbb R}}

\def\PP{{\mathbb P}}

\def\EE{{\mathbb E}}
\def\XX{{\mathbb{X}}}
\def\YY{{\mathbb{Y}}}

\def\Um{{U_*}}

\spnewtheorem{assumption}[theorem]{Assumption}{\bf}{}

\newcommand{\JB}[1]{#1}
\newcommand{\GB}[1]{#1}

\usepackage[backend=biber, style=authoryear-comp, giveninits=true, uniquename=false, url=false, isbn=false, doi=false, dashed=false,  natbib=true, maxcitenames=2, maxbibnames=10, date=year]{biblatex}

\DeclareNameAlias{sortname}{last-first} 
\DeclareFieldFormat{pages}{#1} 

\renewbibmacro*{volume+number+eid}{%
  \printfield{volume}%
  \setunit*{\addnbspace}
  \printfield{number}%
  \setunit{\addcomma\space}%
  \printfield{eid}}
\DeclareFieldFormat[article]{number}{\mkbibparens{#1}}
\renewbibmacro{in:}{%
  \ifboolexpr{%
     test {\ifentrytype{article}}%
     or
     test {\ifentrytype{inproceedings}}%
  }{}{\printtext{\bibstring{in}\intitlepunct}}%
}

\DeclareFieldFormat[article,inbook,incollection,inproceedings,patent,thesis,unpublished]{citetitle}{#1}
\DeclareFieldFormat[article,inbook,incollection,inproceedings,patent,thesis,unpublished]{title}{#1}

\begin{document}
\mainmatter              

\title{Nonasymptotic \GB{one-} and two-sample tests in high dimension \GB{with unknown covariance structure}}

\author{Gilles Blanchard \and Jean-Baptiste Fermanian}

\titlerunning{Nonasymptotic two-sample tests}  
\author{Gilles Blanchard\inst{1,2} \and Jean-Baptiste Fermanian\inst{1,3}}
\authorrunning{G. Blanchard and J-B. Fermanian} 

\institute{Institut de Mathématiques d'Orsay, CNRS, Université Paris-Saclay
\and Inria \and
École Normale Supérieure de Rennes}

\maketitle              

\begin{abstract}
    Let $\XX = (X_i)_{1\leq i \leq n}$ be an i.i.d. sample of square-integrable variables in $\mathbb{R}^d$, \GB{with common expectation $\mu$
      and covariance matrix $\Sigma$, both unknown.} We consider the problem of testing if $\mu$ is $\eta$-close to zero, i.e. $\|\mu\| \leq \eta $ against $\|\mu\| \geq (\eta + \delta)$; we also tackle the more general two-sample mean closeness
    (also known as {\em relevant difference}) testing problem. The aim of this paper is to obtain nonasymptotic upper and lower bounds on the minimal separation distance
  $\delta$ such that we can control both the Type~I and Type~II errors
  at a given level. The main technical tools are concentration inequalities, first for
  a suitable estimator of $\norm{\mu}^2$ used a test statistic,
  and secondly for estimating the operator and Frobenius norms of
  $\Sigma$ coming into the quantiles of said
  test statistic. These properties are obtained for Gaussian and bounded distributions. A particular attention is given to the dependence in the pseudo-dimension $d_*$ of the distribution,
  defined as $d_* := \norm{\Sigma}_2^2/\norm{\Sigma}_\infty^2$. In particular, for $\eta=0$, the minimum
  separation distance is ${\Theta}( d_*^{\nicefrac{1}{4}}\sqrt{\norm{\Sigma}_\infty/n})$, in
  contrast with the minimax estimation distance for $\mu$, which is
  ${\Theta}(d_e^{\nicefrac{1}{2}}\sqrt{\norm{\Sigma}_\infty/n})$ (where $d_e:=\norm{\Sigma}_1/\norm{\Sigma}_\infty$). This generalizes
  a phenomenon spelled out in particular by \citet{Bar02}.
\end{abstract}

\centerline{\em Contribution to a Festschrift volume in the honor of V. Spokoiny's 60th birthday}

\section{Introduction}

We consider the following fundamental signal detection problem: given an i.i.d. sample $\XX=(X_i)_{1 \leq i \leq n}$ from a \GB{square} integrable
distribution \GB{$\mbp_X$} on $\mbr^d$ (or possibly a separable Hilbert space, under some conditions which will be discussed later)
with $\mu=\e{X_1}$, test the hypothesis of ``$\eta$-closeness to zero'' of the mean:
\begin{equation}
  \label{eq:mmodel}
  (H_0(\eta)) : \| \mu \| \leq \eta, \text{ against }  (H_1(\eta,\delta)) :  \| \mu \| > \eta+\delta.
\end{equation}
In fact, we  consider the following more general
two-sample mean closeness testing problem:
for $\XX=(X_i)_{1 \leq i \leq n}$ and $\YY = (Y_i)_{1 \leq i \leq m}$ two independent samples of i.i.d. variables
with distributions \GB{$\mbp_X,\mbp_Y$} on $\mbr^d$ with respective means $\mu$ and $\nu$, test the hypothesis
of $\eta$-closeness (or similarity) of the two means,
\begin{equation}\label{eq:mmodel_2sample}
  (H_0(\eta)) : \| \mu - \nu \| \leq \eta, \text{ against }  (H_1(\eta,\delta)) :  \| \mu - \nu \| > \eta+\delta.
\end{equation}
Observe that we can always formally subsume setting~\eqref{eq:mmodel} into setting~\eqref{eq:mmodel_2sample}, by letting $m$ go to infinity \GB{and/or}
assuming (if needed) that the covariance of $Y_1$ is zero. Therefore, in the
contribution section
we will concentrate mainly on setting~\eqref{eq:mmodel_2sample}.

The problem~\eqref{eq:mmodel} (and numerous extensions thereof) has been \JB{
a long-time subject of attention in mathematical statistics. For the zero mean test problem, i.e. $\eta = 0$,
the celebrated works of \citet{Ing82,Ing93} in the Gaussian white noise model are seminal. In the case $\eta >0$, the problems~\eqref{eq:mmodel}-\eqref{eq:mmodel_2sample} are known as testing for
{\em precise hypotheses}, {\em relevant hypotheses} or {\em relevant differences} \citep{BerDel87};
this setting has found applications in particular in biostatistics for bioequivalence testing (see e.g. \citealp{wellek2002}).
(See next sections for a more detailed discussion of related literature.) }
In this work, we will consider the situation where the involved distributions are either Gaussian
or of bounded norm (and hence sub-Gaussian), \GB{but with unknown covariance matrix acting
  as a nuisance parameter.}

We are interested in finding bounds on the separation distance $\delta$, i.e.
a bound on the minimum value of $\delta$ such that there exists a test with both Type I and Type II error
rates bounded by a ``small'' prescribed quantity. Our interest here is more on the constructive side, so that
we will concentrate on feasible procedures that are in particular adaptive to
\GB{the covariances of the involved distributions.}
A matching lower bound (for \GB{any} fixed covariance structure) will be provided in the Gaussian setting.
We emphasize that our focus is on {\em finite sample}
(i.e. nonasymptotic) results, as will be discussed below.

\subsection{Relation to white noise model in nonparametric statistics}
\label{se:relnpstat}

In the isotropic Gaussian case (white noise) \GB{with known variance} , and for $\eta=0$, the
signal detection problem~\eqref{eq:mmodel} has been studied in much generality,
in particular in the infinite-dimensional setting where $\mbr^d$ is replaced by a separable Hilbert space.
In this situation, due to the fact that the white noise model on an infinite-dimensional Hilbert space
cannot be represented by a random variable taking values in that space, the canonical model which is
considered instead is the
Gaussian sequence model for the coordinates of each of the observations in an orthonormal basis
(in fact the Gaussian sequence model \GB{with known variance}
is usually considered with a single observation of the sequence):
\begin{equation}
  \label{eq:gaussseq}
  X^{(i)} = \mu^{(i)} + \sigma \eps^{(i)}, \qquad i \in \mbn_{>0},
\end{equation}
where $(\eps^{(i)})_{i\geq 1}$ is an i.i.d. standard normal sequence.
This fundamental model in nonparametric statistics  allows to represent in
a clean way many functional spaces of interest for the signal $\mu$ through
geometrical properties of its expansion coefficients $(\mu^{(i)})_{i\geq 1}$ in a suitable
basis. Since in that infinite-dimensional setting the alternative
$\norm{\mu}^2>\delta^2$ is ``too big'' and gives rise to trivial separation rates,
the usual focus is on considering restricted alternatives of the form $\set{\mu \in \cF; \norm{\mu} \geq \delta^2}$,
for a given nonparametric set $\cF$. Classical alternatives of interest include in particular $\ell_2$ ellipsoids
(corresponding to Hilbert norms of different strengths), $\ell_p$ bodies, and Besov bodies.
\GB{Interpreted in functional spaces,} these alternatives correspond respectively to balls in Sobolev spaces (typically when considering Fourier basis
coefficient expansions) or in Besov spaces (for suitable wavelet basis coefficient expansions).

The literature on these topics is profound and extensive, see e.g. \citet{IngSus12} for a comprehensive overview.
The case of certain classes of $\ell_2$-ellipsoids appears to have
been studied first by \citet{Ing82} and \citet{Erm91}, then a remarkable series of works of \citet{Ing93,IngSus98}
established minimax testing rates for general $\ell_2$ ellipsoids as well as other alternatives.
V.~Spokoiny's contribution is prominent in this body of literature, in particular for dealing with the case of Besov bodies
\citep{LepSpo99} as well as considering the problem of statistical adaptivity over a family of alternatives
\citep{Spo96}.

This very limited overview of the topic of testing in the white noise model is meant
to contrast with the setting considered here. On the one hand, we will not consider a particular form of alternative;
on the other hand, we assume that the observations can truly be represented as elements in a possibly infinite-dimensional
separable Hilbert space. Under the Gaussian assumption, this means that the covariance operator $\Sigma$ of the noise
process is assumed to have a finite trace, which also prevents the triviality problem mentioned above for the
white Gaussian noise setting. If we represent the observation coordinates in a diagonalizing basis of $\Sigma$,
our setting in the Gaussian setting amounts to the Gaussian sequence model~\eqref{eq:gaussseq} wherein the
constant parameter $\sigma$ is replaced by a square integrable sequence $(\sigma^{(i)})_{i\geq 1}$.
Note that formally normalizing the $i$-th observation coordinate by $\sigma^{(i)}$ would give rise again to model~\eqref{eq:gaussseq},
however the separation distance would then be measured in the weak norm $\norm[1]{\Sigma^{\nicefrac{1}{2}} \mu}$.

\subsection{Relation to \GB{``modern''} and high-dimensional statistics}
Since we only consider test separation distance without a specific alternative, the setup we consider can be considered as less elaborate,
at least in the sense of asymptotic theory, than the settings with various non-parametric alternatives discussed above. On the other hand, our focus is
specifically on the following points:
\begin{enumerate}
\item Finite-sample analysis;
\item Non-Gaussian data (we will only consider bounded data here);
\item Robustness to misspecification (here under the form of the relaxed composite null $\norm{\mu}^2 \leq \eta^2$,
  also called {\em relevant hypothesis testing}).
\end{enumerate}
These features have been rightly identified by V.~Spokoiny as the defining features of ``modern'' approach to statistics \citep{Spo12,SpoDic15}.
The problem of testing a null hypothesis defined as a neighborhood rather than an exact match
  has been tackled under different settings in the statistics literature, especially for
  the two-sample testing case. For example, motivated by bioequivalence testing between
  populations, \citet{MunCza98} consider the problem of testing
  closeness of two real distributions as measured in Mallows distance, \citet{DetMun98}
  that of closeness in $L^2$ distance of two nonparametric (H\"older regular) regression functions;
  \citet{DetKokAue20}, the closeness in supremum norm distance of two mean functions in a Banach functional data setting;
  \citet{DetKokVol20}, the closeness in $L^2$ distance of the functional mean of time series.
  In all cases, the underlying principle is to estimate the target distance --- as will be also case in the present paper --- and the data is not always assumed to be Gaussian,
  but the corresponding analysis based on Gaussian asymptotic theory. 
  \JB{To estimate the quantiles of the test statistic, \citet{DetMun98} choose to estimate the variance, \citet{DetKokVol20}
    use a self-normalized procedure and give asymptotic bounds; \citet{DetKokAue20}
    propose a bootstrap approach and obtain an asymptotic convergence of the test statistic. In the present
  paper our approach is a direct estimation of the variance with nonasymptotic guarantees.}

\GB{Taking the above aspects into account in the theory, in particular non-asymptotic analysis,} is motivated by a large number of high-dimensional applications, where it appears that relying
on traditional asymptotic of Gaussian parametric or non-parametric theory can possibly be problematic if done without care. Finite sample theory
allows to delineate more precisely in which situations traditional approximations still can be relied upon, and to study non-standard asymptotics,
in particular when key parameters, such as dimensionality, can themselves depend on the sample size~$n$. It is also of use when considering
multiple testing scenarios, where multiplicity has to be taken into account
precisely.

Another fruitful modern insight is that high-dimensional statistical
models tend to blur the line between parametric and non-parametric
point of views.  Precise non-asymptotic results in a
finite-dimensional setting, but where the role of key model parameters
(in particular, dimensionality or effective dimensionality) is
precisely analyzed, can provide key theoretical components for
analyzing non-parametric settings. In the signal testing framework
considered in the present paper, this way of thinking has in particular been pioneered by
\citet{Bar02}, who obtained sharp non-asymptotic results for the problem~\eqref{eq:mmodel} in the case $\eta=0$, and for the finite-dimensional
counterpart of the white noise model~\eqref{eq:gaussseq}, i.e. the isotropic setting $\Sigma= \JB{\sigma^2} I_d$
in dimension $d$.
Baraud further demonstrated that this result provided a valuable and versatile tool to analyze models
of typical interest in high-dimensional statistics (such as sparse alternatives) as well as non-parametric
alternatives (such as those mentioned in the previous section). A key insight from Baraud's
work is that the minimum  separation distance in that setting  is
$\mtc{O}( d^{\nicefrac{1}{4}}\sigma/\sqrt{n})$, in
  contrast with minimax estimation distance for $\mu$, which is
  ${\Theta}(d^{\nicefrac{1}{2}}\sigma /\sqrt{n})$: the testing separation distance is smaller than the minimax estimation
  error by a factor $d^{\nicefrac{1}{4}}$.

  \GB{Analyzing precisely the role of
    dimensionality (ambient or effective) in minimax testing separation rates and
    the difference with minimax estimation rates has been a subject of interest in recent
    literature in various settings, highlighting similar related phenomena.
    For instance, \citet{Lam21} consider the problem of testing equality of two
    high-dimensional multinomial distributions and study the minimum $\ell_1$ separation
    distance in a vicinity of a reference distribution $\pi$ (which implicitly determines a notion
    of local effective dimensionality). Since this model has bounded data, our analysis could
    be applied in that setting, however it concerns separation in $\ell_2$ distance (the separation in
    $\ell_1$ distance exhibits considerably more involved behavior).
    \citet{Ost20} consider a different type of two-sample testing problem, in a regression
    context, where the
    goal is to determine which one of the two distributions has a given (known to the user) regression vector.
    They give a sharp bound on the minimum separation distance between the two regression
    vectors including the role of the dimension, also exhibiting a difference with estimation rates.
  }

\GB{Coming back to our model,  the results of \citet{Bar02} provide a sharp answer}, but
only in the case $\eta=0$ and for isotropic Gaussian (white noise) data \GB{with known variance}.
Still in the Gaussian isotropic case, the minimum
separation rates for any value of $\eta\geq 0$ were precisely
characterized by \citet{BlaCarGut18}. We also
consider the Gaussian setting in the present work, but analyze the generalized situation
where the covariance matrix $\Sigma$ can be arbitrary \GB{(and unknown)}.
In this situation, the role of the dimensionality $d$ is played by proxy quantities
depending on $\Sigma$, sometimes called effective dimensionality or effective rank.
For the signal testing problem however, it turns out that the proxy dimensionalities
for testing and estimation differ.
Namely, for $\eta=0$, we find that the minimax
separation distance is $\mtc{O}( d_*^{\nicefrac{1}{4}}\sqrt{\norm{\Sigma}_\infty/n})$, where
$d_* := \norm{\Sigma}_2^2/\norm{\Sigma}_\infty^2$, while
  the minimax estimation distance for $\mu$ is
  ${\Theta}(d_e^{\nicefrac{1}{2}}  \sqrt{\norm{\Sigma}_\infty/n})$, where $d_e:=\norm{\Sigma}_1/\norm{\Sigma}_\infty$.
  (Notice that $d_* \leq d_e \leq d$ in general, \GB{while these quantities} are all equal in the isotropic setting.)
 Furthermore, we also study the estimation of key quantities
$\norm{\Sigma}^{\nicefrac{1}{2}}_{\infty}$ and $\norm{\Sigma}_2$ determining the proxy dimensionality and the
testing threshold\footnote{With the notation $\norm{\Sigma}_p$ we mean $p$-Schatten norm. We will
  freely use in the paper the equivalent notation $\norm{\Sigma}_\infty = \norm{\Sigma}_{\mathrm{op}}$,
$\norm{\Sigma}_1 = \tr(\Sigma)$, $\norm{\Sigma}_2^2 = \tr(\Sigma^2)$.}.

A crucial mathematical tool in high-dimensional statistics is to obtain sharp concentration
inequalities for quadratic forms of random vectors.
These are closely related to technical tools used in the present work.
An important point \JB{in} such inequalities is to quantify as precisely as possible
up to which point quadratic forms of non-Gaussian
vectors can mimic \JB{the} Gaussian behavior (i.e. that of central and non-central weighted chi-squared
statistics). This topic has received a good deal of attention in the recent years and V.~Spokoiny
also made substantial contributions to that area \citep{SpoZhi13,SpoDic15}. In the present work,
we derive from scratch the needed concentration inequalities; we discuss in more detail
the relation to V.~Spokoiny's own work and to related literature in Section~\ref{sse:discsubg}.

\subsection{Relation to machine learning and kernel mean embeddings
of distributions}

\label{se:kme}

An application setting which motivated us to consider in detail the case of bounded data is
that of testing of the data distribution via kernel mean embedding (KME) methods, a principle
which has garnered a lot of attention in the machine literature since the seminal paper of
\citet{SmoGreSonSch07}. It has been advocated in particular for two-sample \citep{Greetal12} and goodness-of-fit
\citep{ChwStrGre16} testing; see \citet{MuaFukSriSch17} for a recent overview.

We describe the KME principle briefly.
Assume $Z$ is a random variable with distribution $\mbp_Z$ taking values in the
measurable space $\cZ$, and that one has at hand a fixed mapping $\Phi: \cZ \rightarrow \cH$, where
$\cH$ is a separable Hilbert space. To this mapping is associated a reproducing kernel Hilbert space (rkHs)
$\cH'$ with kernel $k(z,z') := \inner{\Phi(z),\Phi(z')}$.

Assuming the variable $X=\Phi(Z)$ is Bochner integrable\footnote{\GB{that is, the real random variable $\norm{\Phi(Z)}$ is
  integrable, which guarantees that the integral of $\Phi(Z)$ is well-defined in a strong
sense as an element of the Hilbert space; see e.g. \citet{CohDon80}.}}
\JB{(which is the case in particular when the mapping $\Phi$ is bounded)}, the kernel mean embedding of $\mbp_Z$ is defined as $\Phi(\mbp_Z) := \e{\Phi(Z)} \in \cH$
(using a rather natural overload of notation for $\Phi$). The {\em maximum mean discrepancy} (MMD)
between distributions $\mbp,\mbq$ in the domain of definition of $\Phi$ is defined as the semimetric
\[
  \mathrm{MMD}_k(\mbp,\mbq) := \norm{\Phi(\mbp)-\Phi(\mbq)}.
\]
Since $\mathrm{MMD}_k(\mbp,\mbq)>0$ implies $\mbp\neq \mbq$, this principle can be used for
simple goodness-of-fit testing (testing for $\mbp_{Z} = \mbp_0$ for some known distribution
$\mbp_0$, given an i.i.d. sample from $\mbp_Z$) and two-sample testing (testing for $\mbp_{Z}=\mbp_{Z'}$, given two independent i.i.d. samples from $\mbp_{Z}$ and $\mbp_{Z'}$); in each case,
the test statistic is a suitable estimator of  $\mathrm{MMD}_k(\mbp_Z,\mbp_0)$, resp. $\mathrm{MMD}_k(\mbp_Z,\mbp_{Z'})$ from the observed data. More generally one may want to
test the relaxed null hypothesis $\mathrm{MMD}_k(\mbp,\mbq) \leq \eta$ and analyze the power
of the test in terms of the MMD separation itself.
This is indeed a particular case of~\eqref{eq:mmodel}-\eqref{eq:mmodel_2sample}, when considering
the Hilbert-valued variable $X=\Phi(Z)$ and, for two-sample testing, $Y=\Phi(Z')$.

A common situation is when $\Phi$ is bounded in norm by some constant $L$, or
equivalently in terms of the kernel, $\sup_{z\in \cZ} k(z,z) \leq L^2$. This ensures in
particular that $\Phi$ is defined on all distributions.
Analyzing our original setting with norm-bounded but potentially infinite-dimensional data
is therefore  suited to this case.

\citet{Greetal12} derive the asymptotic distribution of the (suitably renormalized) MMD test statistic,
which is identical to the one we use below (once interpreted in the KME setting).
Unsurprisingly, a Gaussian limiting behavior is
identified. Our study analyzes this behavior from a non-asymptotic point of view; this can be
particularly of interest for situation where the mapping $\Phi$ (or equivalently the associated kernel)
is to depend  on the sample size,
or when performing a large number of such tests in parallel: in this case
uniformly valid nonasymptotic bounds are a a valuable tool for further analysis. See \citet{MarBlaFer20} for
such a multiple test scenario in the context of so-called multiple task averaging.
\GB{Multiple tests can also be aggregated to test a global hypothesis, see \citet{Fro12}
in the context of two-sample testing based on the KME approach.}

In our study, the power of the test is investigated for alternatives of the
form~\eqref{eq:mmodel}-\eqref{eq:mmodel_2sample}, which, interpreted in the KME setting,
correspond to $\mathrm{MMD}_k(\mbp,\mbq) \geq \eta + \delta$.
The power of KME-based tests (in the goodness-of-fit case) was also investigated by~\citet{BalYua21}, but for alternatives measured in a $\chi^2$ distance separation,
more precisely,
of the form $\set{\mbq \in \cF; \chi^2(\mbp_0,\mbq) \geq \delta}$, where $\cF$ is a
nonparametric set of distributions whose density with respect to $\mbp_0$ is approximated
at a given rate by functions in the rkHs $\cH'$ associated to $k$, in the sense of
interpolation with $L^2(\mbp_0)$. This is close in spirit to nonparametric points of view
discussed in Section~\ref{se:relnpstat}, in the sense that $\chi^2$-separation alone is
too weak to get nontrivial separation rates and one has to additionally consider
intersection with nonparametric sets of interest. Again, because we choose
to analyze alternatives measured in $\mathrm{MMD}_k$-separation itself, the results we obtain
in this setting have a different nature.

\JB{
\subsection{Overview of contributions}
The main contribution of this paper is to give upper bounds on the optimal (minmax) testing
separation distance for problems~\eqref{eq:mmodel} and~\eqref{eq:mmodel_2sample}
over classes of probability distributions with fixed covariance matrix $\Sigma$ for sample $\XX$, as well as $S$
for sample $\YY$ in the two-sample case. The covariance structures are considered as nuisance parameters
and we investigate precisely how they influence the testing separation distance.
Let $\cP$ be a family of distributions for the two samples (we consider the Gaussian setting and the bounded setting), and
$\cP_{\Sigma,S}$ the subsets of distributions of $\cP$ with $\cov{X_1}=\Sigma$, and $\cov{Y_1}=S$ (in the two-sample case).
Consider the sets of distributions 
\begin{align*}\label{eq:def_A_C}
  \cH_0( \eta , \Sigma,S) & := \{ \mbp \in \cP_{\Sigma,S} | \mbp \text{ satisfies } H_0(\eta) \}\,,\\
  \cA_\delta(\eta , \Sigma,S)& :=  \{ \mbp \in \cP_{\Sigma,S} | \mbp \text{ satisfies } H_1(\eta,\delta) \}\,,
\end{align*}
then the optimal separation distance is, for $\alpha \in (0,1)$:
\begin{equation}\label{eq:def_delta*}
  \delta^*(\alpha, \Sigma,S,\eta ) = \inf\set{ \delta \geq 0 \Big| \exists\, \text{test $T$ }: \sup_{\mbp \in \cH_0} \mbp\paren{ T =1 } + \sup_{\mbp \in \cA_\delta }  \mbp\paren{ T =0 } \leq \alpha }\,.
\end{equation}
In the  Gaussian setting, we establish that $\delta^*$ is upper bounded up to a constant factor via
\begin{equation}
  \delta^*(\alpha,\Sigma,S , \eta) \lesssim \sigma \kappa_\alpha \max\paren{ 1 , \min \paren[2]{ d_*^{\frac{1}{4}}, d_*^{\frac{1}{2}} \frac{\sigma \kappa_\alpha}{\eta} }}\,, \label{eq:result}
\end{equation}
(Theorem~\ref{thm:sigdetgauss}) where $\kappa_\alpha:=\sqrt{-\log(\alpha)}$, and, in the one-sample case, $\sigma^2:=\norm{\Sigma}_{\mathrm{op}}/n$ is a scalar variance factor and $d_*:= \tr \Sigma^2/\norm{\Sigma}_{\mathrm{op}}^2$ a notion of effective dimension.
\GB{In the two-sample case, we obtain also~\eqref{eq:result}, with  $\sigma^2:=\norm{M_{m,n}}_{\mathrm{op}}$, and $d_* := \tr M_{m,n}^2 / \sigma^4$,
  where $M_{m,n}:=(\Sigma/n+S/m)$ (Theorem~\ref{thm:sigdetgauss_2sample}).}
\GB{In the one-sample case, this result can be formulated equivalently in terms
  of {\em sample complexity} $n^*$ needed to detect at given error level $\alpha$
  and separation distance $\delta$
  for problem~\eqref{eq:mmodel}:
\begin{equation}
  n^*(\alpha,\Sigma,S , \eta) \lesssim {\norm{\Sigma}_{\mathrm{op}} \kappa_\alpha}{\delta^{-1}}
  \max\paren{ \delta^{-1} , d_*^{\frac{1}{2}} \paren{\max( \delta, \eta)}^{-1}}. \label{eq:resultsc}
\end{equation}}

This result is established first when assuming that $\Sigma,S$ are known, then we show that it holds as well
when they are unknown (under some mild assumptions on the sample size, see Corollary~\ref{cor:applquant} for
an explicit statement
in the one-sample case and condition \eqref{eq:cond_n} there).
Matching minimax lower bounds are given for one and two-sample problems in the Gaussian setting.
In the bounded setting, we derive upper bounds only, which take the same flavor as~\eqref{eq:result}
under some mild assumptions on the sample
sizes.

}

\subsection{Organization of the paper}

\GB{We present in Section~\ref{se:main_res} our main results. In order to cover both the Gaussian and bounded settings
under the same umbrella, we start in
Section~\ref{sse:general} by a generic result: assuming some suitable concentration for
an estimate $U$ of the squared signal norm $\norm{\mu}^2$ holds (Assumption~\ref{ass:ass1}),
as well as for estimators of its quantiles (Assumption~\ref{ass:ass2}),
for the problems \eqref{eq:mmodel} and \eqref{eq:mmodel_2sample} we propose in Theorem~\ref{thm:main} sufficient conditions on $\delta$ such that we can control the Type I and Type II errors of a test $T$
based on $U$. In the following sections, the Gaussian setting and the bounded setting are considered separately.}
In Section~\ref{sse:conc}, we give  concentration results for $U$ to fulfill Assumption~\ref{ass:ass1}. In Section~\ref{sse:quantest}
we give results to fulfill Assumption~\ref{ass:ass2}, which are related to the estimation
of $\norm{\Sigma}_\infty^{\nicefrac{1}{2}}$ and $\norm{\Sigma}_2$.
The proofs of the corresponding results are found in Sections~\ref{se:prmain} to~\ref{se:tr_sigma2},
respectively.

\section{Main results}\label{se:main_res}

We will build a test for the model \eqref{eq:mmodel_2sample} based on an estimator $U$ of the distance $\|\mu - \nu\|^2$, typically a
modified U-statistic as defined below.
We will first consider a general point of view to deduce bounds on the separation
rate when $U$ satisfies certain concentration properties; this will then apply both to the
Gaussian and bounded settings.

\subsection{A general result to upper bound separation rates}
\label{sse:general}

\GB{As mentioned earlier, from now on we concentrate primarily
  on the two-sample setting, being understood that
    upper bounds for the one-sample setting can be deduced readily.
In order to define a general framework
\GB{encompassing as particular cases the more specific settings considered below, in
  this section} we will assume
a generic statistical model $\cP$
for the distribution of the
samples $\XX$ and $\YY$, which we recall we always
assume to be independent and i.i.d. with respective \GB{squared integrable} marginal distributions $\mbp_X,\mbp_Y$. We will thus use without comment the fact that
a distribution $\mbp \in \cP$ equivalently specifies the marginal distributions $\mbp_X$ and $\mbp_Y$ of the samples. We will consider the covariance matrices $\Sigma,S$ of $\mbp_X,\mbp_Y$ as nuisance
parameters influencing the optimal separation distance, and define the sub-models
\[
  \cP_{\Sigma,S} =\set{\mbp \in \cP: \cov{\mbp_X}=\Sigma, \cov{\mbp_Y} = S};
\]
$\cP_{\Sigma}$ is defined in an analogous way for the one-sample setting.}

The first property we require is a form of 2-sided concentration of $U$ around
the target quantity:
\begin{assumption} \label{ass:ass1}
  For any $(\Sigma,S)$ and distribution $\mbp \in \cP_{\Sigma,S}$;
  for any given $\alpha \in (0,1)$ there exist $q_1=q_1(\Sigma,S,\alpha), q_2=q_2(\Sigma,S,\alpha)$ in $\RR_+$ such that:
  \begin{equation}\label{eq:hyp_conc_U} 
    \prob{ \abs{ U - \|\mu - \nu\|^2 } \geq \|\mu - \nu \| q_1 + q_2 } \leq \alpha\,.
  \end{equation}
\end{assumption}
\GB{Additionally, we will consider the situation where the quantities $q_1$, $q_2$
(which are necessary to find a suitable testing threshold) are not known but
must also be estimated from the data; this is the case if the
covariance matrices $(\Sigma,S)$ are unknown.}
This leads us to our second assumption:
\begin{assumption} \label{ass:ass2}
  Suppose Assumption~\ref{ass:ass1} holds, with the
  notation introduced therein.
  For any $\alpha \in (0,1)$ there exist two estimators $\widehat{Q}_1 = \widehat{Q}_1(\alpha)$ and $\widehat{Q}_2 = \widehat{Q}_2(\alpha)$ \GB{ in $\mbr_+$} such that,
  \GB{for any  $(\Sigma,S)$ and distribution $\mbp \in \cP_{\Sigma,S}$:
    \begin{gather}\label{eq:hyp_conc_q}
      \prob{\abs{ q_1(\Sigma,S,\alpha) - \widehat{Q}_1(\alpha) } \geq \frac{1}{2} q_1(\Sigma,S,\alpha)} \leq \alpha \,, \\
       \prob{ \abs{ q_2(\Sigma,S,\alpha) - \widehat{Q}_2(\alpha) } \geq \frac{1}{2} q_2(\Sigma,S,\alpha) } \leq \alpha\,.
    \end{gather}}
\end{assumption}
\GB{(In the ``oracle'' case where the covariances $\Sigma,S$ are assumed to be known,
  of course Assumption~\ref{ass:ass2} is trivially satisfied taking $\wh{Q}_1=q_1$,
  $\wh{Q}_2=q_2$.)}
The following generic result transforms the above assumptions into an
estimate of the separation distance for setting~\eqref{eq:mmodel_2sample}.
\begin{theorem}\label{thm:main}
\GB{  Let $\cP$ be a statistical model \GB{for setting~\eqref{eq:mmodel_2sample}},
  and $U$ be a statistic. Let Assumptions~\ref{ass:ass1} and~\ref{ass:ass2} be granted.
  Given $\eta\geq 0$ and $\alpha \in (0,1)$, let $T$ be the test defined by
      \begin{equation}\label{eq:def_T}
    T = 1\set{ U - \eta^2 >  2\eta\widehat{Q}_1(\alpha) + 2\widehat{Q}_2(\alpha) }\,.
  \end{equation}
  Then for any $(\Sigma,S)$, provided
     \begin{equation}\label{eq:hyp_separation_rate}
       \delta > 2q_1 + \min\left( 2\sqrt{q_2} ,
       2\eta^{-1}q_2 \right)\,,
    \end{equation}
    it holds, for any distribution $\mbp \in \cP_{\Sigma,S}$:
    \begin{align*}\label{eq:level_power}
      \prob{T=1} & \leq 3 \alpha\,, \text{ if } \mbp \text{ satisfies } (H_0(\eta));\\
      \prob{T=0} & \leq 3 \alpha\,, \text{ if } \mbp \text{ satisfies } (H_1(\eta,\delta)).
    \end{align*}}
  \end{theorem}

  \subsection{Concentration properties of the test statistic}
  \label{sse:conc}
  The rest of the paper is dedicated to establishing the validity of Assumptions~\ref{ass:ass1} and~\ref{ass:ass2} for the following statistic $U(\XX,\YY)$:
     \begin{equation}\label{eq:def_U}
       U(\XX, \YY) := \frac{1}{n(n-1)} \sum_{\substack{i, j=1 \\ i \neq j}}^{n} \inner{ X_i , X_j } + \frac{1}{m(m-1)} \sum_{
         \substack{i,j=1 \\ i \neq j}}^{m} \inner{ Y_i , Y_j } - \frac{2}{nm } \sum_{i=1}^{n} \sum_{j=1}^{m} \inner{ X_i , Y_j }\,.
  \end{equation}
  Observe that provided expectations $\mu,\nu$ exist, $U(\XX,\YY)$ is an unbiased estimator of
  $\norm{\mu-\nu}^2$. In the KME setting as described in Section~\ref{se:kme}, inner products are replaced by
  kernel evaluations and the above statistic is the standard unbiased estimate of the squared MMD
  between $\mbp_X$ and $\mbp_Y$. As announced previously, we will concentrate on the following
  two \JB{settings}:

  \begin{definition}[Gaussian \JB{setting}] The samples $\XX$ and $\YY$ are i.i.d. Gaussian in $\mbr^d$ of marginal distributions
    $\mbp_X=\cN(\mu,\Sigma)$ and $\mbp_Y=\cN(\nu,S)$, respectively.
  \end{definition}
  In the Gaussian setting, we will assume a finite ambient dimension $d$ for technical reasons:
  our proofs rely on the Gauss-Lipschitz concentration inequality, which applies
  in finite dimension. As will appear clearly however, all our results to come are dimension-free
  in the sense that $d$ never enters the picture, instead only norms of $\Sigma,S$
  come into play. We surmise that our results would apply as well in the same form
  in the Hilbert-valued setting provided $\tr(\Sigma)$ and $\tr(S)$ are finite, but did not
  try to write down a precise approximation argument to this end.

  \begin{definition}[Bounded \JB{setting}] The samples $\XX$ and $\YY$ are i.i.d. in a separable Hilbert space $\cH$ with norm bounded by $L>0$.
    The covariance operators for the marginal sample distributions are denoted $\Sigma$ and $S$, respectively;
    observe that they have finite trace by the boundedness assumption.
  \end{definition}

   Propositions \ref{prop:conc_U_Gaussian_case} and  \ref{prop:conc_U_bounded_case} give concentration bounds for the statistic $U$, \GB{ensuring Assumption~\ref{ass:ass1}} in the two above \JB{settings}.

\begin{proposition}\label{prop:conc_U_Gaussian_case}
  Assume the Gaussian \JB{setting} holds and $n,m\geq 2$. Then
 with probability at least \JB{$1-\alpha$,
\begin{equation}\label{eq:conc_U_Gaussian_case}
  \abs{ U - \|\mu- \nu \|^2 } \leq \|\mu- \nu\|q_1 + q_2\,,
\end{equation}
where $U$ is defined in \eqref{eq:def_U} and
\begin{align}
  q_1(\Sigma,S,\alpha) &= \sqrt{2\paren{\frac{\|\Sigma\|_{\mathrm{op}}}{n} + \frac{\|S\|_{\mathrm{op}}}{m}}u(\alpha)}\,, \label{eq:defq1} \\
  q_2(\Sigma,S,\alpha) &= 32\paren{ \frac{\sqrt{\tr \Sigma^2}}{n} + \frac{\sqrt{\tr S^2}}{m} } u(\alpha) \,. \label{eq:defq2}
\end{align}
where $u(\alpha) := -\log\alpha + \log 8$.
}
\end{proposition}
Let us simplify somewhat the above expression when plugged into Theorem~\ref{thm:main} in the case of signal
detection~\eqref{eq:mmodel}. We also give a matching lower bound (up to constant factor)
for the optimal separation
distance.
\begin{theorem} \label{thm:sigdetgauss}
  \JB{ Consider the signal detection problem~\eqref{eq:mmodel} and assume the Gaussian \JB{setting}
    with covariance matrix $\Sigma$.
    Then the minimum separation distance $\delta^*$ given by~\eqref{eq:def_delta*} so that the type I and II errors for problem~\eqref{eq:mmodel} are less that $\alpha\in (0,1)$  for all $\mbp \in \cP_\Sigma$ is upper bounded by
  \begin{equation}
    \label{eq:signalsep}
    \delta^*(\alpha, \Sigma ,\eta)
       \lesssim   \sigma_n \sqrt{u } \max\paren{1,\min
      \paren{d_*^{\frac{1}{4}},\sqrt{d_*u} \cdot{\frac{\sigma_n}{\eta}}}},
  \end{equation}
  where $u(\alpha) := -\log \alpha + \log 60$. If $d_* \geq 3$, then it is lower bounded by
  \begin{equation}\label{eq:signalsep_low}
   \delta^*(\alpha,\Sigma, \eta ) \geq \sigma_n \sqrt{\frac{1-\alpha}{12}}\max\paren{1 ,\min \paren{d_*^{\frac{1}{4}}, \sqrt{d_*(1-\alpha)} \cdot \frac{\sigma_n}{\eta} }}\,,
  \end{equation}
  where $\sigma^2_n:=\norm{\Sigma}_{\mathrm{op}}/n$, and $d_* := \tr \Sigma^2 / \norm{\Sigma}_{\mathrm{op}}^2$. (The symbol $\lesssim$ indicates inequality up to a numerical factor).
  }
\end{theorem}
Observe that it holds $d_* \leq d_e$, where $d_e=\tr\Sigma/\sigma^2$ is the ``effective dimensionality'' coming into play for signal estimation rates
(namely $\e[1]{\norm[0]{\ol{X}-\mu}^2}^{\nicefrac{1}{2}} = \sigma \sqrt{d_e/n}$, where $\ol{X}$ is the empirical mean).
In the finite $d$-dimensional case with $\Sigma=I_d$, it holds
$d=d_e=d_*$, and the separation~\eqref{eq:signalsep} has been shown to be
optimal in the Gaussian \JB{setting} for $\eta=0$ by \citet{Bar02} and
for any $\eta\geq 0$ by \citet{BlaCarGut18}. It exhibits a
continuous transition between the signal detection setting
($\eta=0$, $\delta^*\simeq d^{\nicefrac{1}{4}}\sigma/\sqrt{n}$)
and the hyperplane testing setting (which is equivalent to the 1-dimensional
setting by rotational invariance; $\eta \rightarrow \infty$, $\delta^* \simeq
\sigma/\sqrt{n}$).
In that particular situation, we observe that the signal separation distance
is smaller by a factor $d^{\nicefrac{1}{4}}$ than the signal estimation error, a
phenomenon typical of high-dimensional statistics. In the more general setting
studied here where $\Sigma$ can be arbitrary, this difference between
rates can be all the more marked, since in addition $d_*$ can be
much smaller than~$d_e$.

We obtain a similar result for the two-sample problem:
\begin{theorem} \label{thm:sigdetgauss_2sample}
  \GB{Consider the two-sample mean problem~\eqref{eq:mmodel_2sample} and assume the Gaussian \JB{setting}
  with covariance matrices $\Sigma,S$. 
  Then the minimum separation distance $\delta^*$ so that
  the type I and II errors for problem~\eqref{eq:mmodel_2sample} is less than
  $\alpha\in (0,1)$ for all $\mbp \in \cP_{\Sigma,S}$ is upper bounded by
  \begin{equation}
    \label{eq:signalsep_2sample}
    \delta^*(\alpha,\Sigma,S,\eta)
 \lesssim   \sigma_{n,m} \sqrt{u} \max\paren{1,\min
      \paren{d_*^{\frac{1}{4}},\sqrt{d_*u} \cdot{\frac{\sigma_{n,m}}{\eta}}}},
  \end{equation}
  where $u:=-\log \alpha + \log 60$. If $d_* \geq 3$, then it is lower bounded by
  \begin{equation}\label{eq:signalsep_low_2sample}
   \delta^*(\alpha,\Sigma,S,\eta ) \geq \sigma_{n,m} \sqrt{\frac{1-\alpha}{48}}\max\paren{1 ,\min \paren{ d_*^{\frac{1}{4}}, \sqrt{d_*(1-\alpha)} \cdot \frac{\sigma_{n,m}}{\eta} }}\,,
  \end{equation}
  where $\sigma_{n,m}^2:=\norm{M_{n,m}}_{\mathrm{op}}$, and $d_* := \tr M_{n,m}^2 / \sigma_{n,m}^4$,
  for $M_{n,m} := \Sigma/n+S/m$. (The symbol $\lesssim$ indicates inequality up to a numerical factor).}
\end{theorem}

\JB{Here the effective dimension $d^*$ depends on the two covariance matrices $\Sigma$ and $S$, weighted by the size of the samples. }

\GB{{\bf Remark.} As mentioned in the introduction, by letting $m$ go to infinity in the two-sample case, we recover the bounds of the one sample case (up to a constant factor). It is worth examining if the converse holds, i.e. if there is an argument to reduce the
  two-sample problem to the simpler one-sample case (this would simplify some technical aspects of the proofs, somewhat). For the {\em upper} bounds on the minimum separation distance,
  this is the case in some specific situations: for equal sample sizes $n=m$,
  the two-sample case can be reduced to the one-sample problem setting by pairing the samples and
  considering the single sample $(X_i - Y_i )_{1\leq i \leq n}$,
  and one can recover this way in essence the two-sample result. If $\Sigma= S$, and for general sample sizes,
  we can also reduce to the single sample with size $\min(m,n)$,
  $(X_i-Y_i)_{1\leq i \leq \min(m,n)}$, and recover again the two-sample result up to a numerical factor.
  However a reduction argument in the general case has eluded us. Concerning the {\em lower} bound,
  the argument for the two-sample case indeed hinges on a reduction the one-sample case,
  by considering the sub-models where one of the two sample means is known, see Section~\ref{se:proofth3}.}

We now turn to the bounded setting.

\begin{proposition}\label{prop:conc_U_bounded_case}
  Assume the bounded \JB{setting} holds and $n,m\geq 2$. Then with probability at least $1-\alpha$,
\begin{equation}\label{eq:conc_U_bounded_case}
  \abs{ U - \|\mu- \nu \|^2 } \leq \|\mu- \nu\|q_1 + q_2\,,
\end{equation}
where $U$ is defined in \eqref{eq:def_U} and
\begin{gather*}
  q_1(\Sigma,S,\alpha)  =  2\sqrt{2 \paren{\frac{\|\Sigma\|_{\mathrm{op}}}{n} + \frac{\|S\|_{\mathrm{op}}}{m} } u } + \frac{4Lu}{3(n \wedge m)}\,, \\
  q_2(\Sigma,S,\alpha)  =   614\paren{ \frac{\sqrt{\tr \Sigma^2 }}{n} + \frac{\sqrt{\tr S^2}}{m}}  u + 3708\frac{L^2 u^2}{(n\wedge m )^2}\,,
\end{gather*}
with $u(\alpha) = -\log \alpha + \log 2$.

\end{proposition}

Thus, in the bounded \JB{setting} we can guarantee that the behavior of the test
is qualitatively the same as in the Gaussian \JB{setting} (see e.g. Theorem~\ref{thm:sigdetgauss}) --- and this from a
non-asymptotic point view, provided $n\wedge m \geq u L^2/\sigma^2$,
where $\sigma^2=\norm{\Sigma}_{\mathrm{op}}$.

A special case of interest is when the data lies on
the sphere of radius $L$, i.e. $\norm{X_i}=\norm{Y_j}=L$ a.s.
In this case $L^2 = \tr \Sigma$ and the above condition
can be rewritten $n \wedge m \geq u d_e$.
This situation is met in particular in the KME setting, see Section~\ref{se:kme},
when using a translation-invariant kernel $k(z,z') = k_\circ(z-z')$, in which
case $L^2 = k_\circ(0)$.

\subsection{Quantile estimation}
\label{sse:quantest}

Since we are considering the case where $\Sigma,S$ can be arbitrary in this
work, it is natural to assume that these are not known in advance.
We study next the estimation of the quantities $q_1$ and $q_2$, in both settings (bounded and Gaussian), in order to check Assumption~\ref{ass:ass2} for our generic
theorem. If we can grant that assumption, Theorem~\ref{thm:main} guarantees that
the separation distance remains qualitatively the same as in the ``oracle''
situation where they are known. \GB{To simplify the exposition, in this section
  we will present results for the one-sample problem only; similar results, although slightly
  more technical, can be obtained
  for the two-sample problem.}
Thus, we need to have estimators of $\|\Sigma\|_{\mathrm{op}}$ and $\tr \Sigma^2$ ---
more precisely, of their square root.

For $q_1$,
we will use the empirical covariance operator $\widehat{\Sigma} := \widehat{\Sigma}(\XX)$:
\begin{equation}\label{eq:def_estimator_sigma_op}
  \widehat{\Sigma}(\XX) = \frac{1}{n} \sum_{i=1}^{n} (X_i - \widehat{\mu} ) (X_i - \widehat{\mu} )^T\,,
\end{equation}
where $\widehat{\mu} := \widehat{\mu}(\XX)$ is the empirical mean of the sample $\XX$.

\begin{proposition}[Gaussian setting] \label{prop:conc_sigma_hat_gauss}
  Assume $\XX = (X_i)_{1\leq i\leq n}$ are i.i.d. Gaussian vectors of covariance $\Sigma$.
  For $ u \geq 0$, with probability at least $1-3e^{-u}$:
\begin{equation}\label{eq:conc_sigma_hat}
  \abs{ {\norm[1]{\widehat{\Sigma}}^{\frac{1}{2}}_{\mathrm{op}}} - {\norm{\Sigma}^{\frac{1}{2}}_{\mathrm{op}}} } \leq 3\sqrt{2}\norm{\Sigma}^{\frac{1}{2}}_{\mathrm{op}} \paren{\sqrt{\frac{d_e}{n}} +  \sqrt{ \frac{u}{n} }}\,,
\end{equation}
where $\widehat{\Sigma}$ is defined in \eqref{eq:def_estimator_sigma_op} and $d_e = \tr \Sigma/\|\Sigma\|_{\mathrm{op}}$.
\end{proposition}

\begin{proposition}[Bounded setting] \label{prop:conc_sigma_hat_bounded}
  Assume that $\XX = (X_i)_{1\leq i\leq n}$ are i.i.d. bounded in norm by $L$ and with covariance $\Sigma$.
  For $ u \geq 0$, with probability at least $1-2e^{-u}$:
\begin{equation}\label{eq:conc_sigma_hat_bounded}
  \abs{ {\norm[1]{\widehat{\Sigma}}^{\frac{1}{2}}_{\mathrm{op}}} - {\norm{\Sigma}^{\frac{1}{2}}_{\mathrm{op}}} } \leq 4L\paren{2 \sqrt{\frac{d_e}{n}} + \sqrt{\frac{2u}{n}} + \frac{u}{3n} }
\end{equation}
where $\widehat{\Sigma}$ is defined in \eqref{eq:def_estimator_sigma_op} and $d_e = \tr \Sigma/\|\Sigma\|_{\mathrm{op}}$.
\end{proposition}

These concentration bounds are not sharp in an asymptotic sense, where the main term for the
scaling of the deviations is expected to follow that of asymptotic normality
for eigenvalues of the empirical covariance
operators, as in the classical results of \citet{And03}, but they are largely sufficient
for our purposes (see Corollary~\ref{cor:applquant} below).
Some refined related nonasymptotic bounds can be found in the recent literature.
In particular, \citet{Kol17} derive nonasymptotic results for controlling
$\norm[1]{\wh{\Sigma}-\Sigma}$ in the Gaussian \JB{setting}, and
in the centered case where $\mu=0$ is known. In fact, in essence the result of our technical
Proposition~\ref{prop:conc_sigma_tilde_gauss} in the proof section (which is like Proposition~\ref{prop:conc_sigma_hat_bounded}
but in the centered case) can be deduced from the results of~\citet{Kol17}
by elementary arguments. We decided to include a standalone proof here; while we do rely
on the estimates of~\citet{Kol17} (or rather on the improved version of
\citealp{Han17}) for the expectation of the difference, we derive an upper bound on
the deviation by a rather direct application of the Gauss-Lipschitz concentration.
While~\citet{Kol17} also rely on such arguments, their proofs are much more involved,
for the reason that they study the norm or the difference while we only are interested
in the difference of the (root) norms here. Finally, we also mention very
recent results of~\citet{JirWah18} for sharp nonasymptotic control of spectral
quantities related to $\Sigma$, which could also potentially be applied here, though
it seems at first glance that a logarithmic dependence in the dimension could enter into play.

For the bounded \JB{setting} (Proposition~\ref{prop:conc_sigma_hat_bounded}), the bound
\eqref{eq:conc_sigma_hat_bounded} could presumably be improved to have $\sqrt{\|\Sigma\|_{\mathrm{op}}}$ instead of $L$ for the main terms. The results of Theorem~9 of \citet{Kol17} under a sub-Gaussian assumption do not seem to be able to imply Proposition~\ref{prop:conc_sigma_hat_bounded}, see the more detailed discussion below in Section~\ref{sse:discsubg}.

Turning now to $q_2$, we will estimate $\sqrt{\tr \Sigma^2}$ using the following
statistic $\widehat{T} := \widehat{T}(\XX)$, which is an unbiased estimator of $\tr \Sigma^2$:
\GB{\begin{equation} \label{eq:def_t_trsigma2}
\widehat{T}(\XX) :=  \frac{1}{4n(n-1)(n-2)(n-3)} \sum_{i \neq j \neq k \neq l } \inner{ X_i-X_k ,X_j- X_l}^2\,.
\end{equation}}
\begin{proposition}[Gaussian setting] \label{prop:conc_sqrt_t_trsigma2}
    Assume $\XX = (X_i)_{1\leq i\leq n}$ are i.i.d. Gaussian vectors of covariance $\Sigma$ and $n\geq 4$.
    Then for all $u \geq 0 $:
\begin{equation}\label{eq:conc_sqrt_t_trsigma2}
      \prob{ \abs{\sqrt{\widehat{T}} - \sqrt{\tr \Sigma^2} } \geq  \GB{30} \sqrt{\frac{\tr \Sigma^2}{n}}u^2 }   \leq e^4e^{-u}\,,
\end{equation}
where $\widehat{T}$ is defined in \eqref{eq:def_t_trsigma2}.
\end{proposition}

\begin{proposition}[Bounded setting] \label{prop:conc_sqrt_t_trsigma2_bounded}
  Assume that $\XX = (X_i)_{1\leq i\leq n}$ are i.i.d. bounded in norm by $L$ and with covariance $\Sigma$ and $n\geq 4$.
Then for all $u \geq 0 $:
\begin{equation}\label{eq:conc_t_trsigma2_sqrt}
      \prob{ \abs{\sqrt{\widehat{T}} - \sqrt{\tr \Sigma^2} } \geq \GB{12}L^2 \sqrt{\frac{u}{n }} } \leq 2 e^{-u}\,.
\end{equation}
where $\widehat{T}$ is defined in \eqref{eq:def_t_trsigma2}.
\end{proposition}

Thanks to these concentration results, we can construct estimators of $q_1(\Sigma ,\alpha)$ and $q_2(\Sigma , \alpha)$ satisfying Assumption~\ref{ass:ass2}.
\GB{In the Gaussian setting, we give the following explicit corollary of Propositions~\ref{prop:conc_sigma_hat_gauss} and~\ref{prop:conc_sqrt_t_trsigma2}; the proof is straightforward and omitted.}

\begin{corollary}[Gaussian setting] \label{cor:applquant}
  Consider the signal detection problem~\eqref{eq:mmodel} and assume the Gaussian \JB{setting} holds. Let $\alpha\in (0,1)$, \GB{$u=u(\alpha) = -\log \alpha + \log 8$},
  and  $\widehat{Q}_1(\alpha)$ and $\widehat{Q}_2(\alpha)$ be the statistics defined by
\begin{equation*}
  \widehat{Q}_1(\alpha) = \sqrt{\frac{2\norm[1]{\widehat{\Sigma}(\XX)}_{\mathrm{op}}}{n} u } \,, \quad \widehat{Q}_2(u) = 32\frac{\sqrt{\widehat{T}(\XX)}}{n}u\,,
\end{equation*}
where $\widehat{\Sigma}$ is defined in \eqref{eq:def_estimator_sigma_op} and $\widehat{T}$ in \eqref{eq:def_t_trsigma2}. Then for any $\Sigma$, provided
\begin{equation}\label{eq:cond_n}
  \GB{n \gtrsim} 
  \max\paren{  d_e(\Sigma), u , u^4}\,,
\end{equation}
(we recall $d_e (\Sigma)= \tr \Sigma/\|\Sigma\|_{\mathrm{op}}$), then it holds, for any distribution $\mbp \in \cP_\Sigma$:
\begin{gather*}
  \prob{ \abs{ \widehat{Q}_1(\alpha) - q_1(\Sigma,\alpha) } \leq q_1(\Sigma,\alpha)/2 } \leq  \alpha, \\
   \prob{ \abs{ \widehat{Q}_2(\alpha) - q_2(\Sigma,\alpha) } \leq q_2(\Sigma,\alpha)/2 } \leq \alpha,
 \end{gather*}
\GB{ where $q_1,q_2$ are as defined in~\eqref{eq:defq1},\eqref{eq:defq2} (with $m=\infty$).}
\end{corollary}

\GB{The condition \eqref{eq:cond_n} for $n$ is needed to grant Assumption~\ref{ass:ass2}:
it ensures that the deviations of the estimators $\norm[0]{\widehat{\Sigma}}_{\mathrm{op}}^{\nicefrac{1}{2}}$ and $\widehat{T}$
coming from Proposition~\ref{prop:conc_sigma_hat_gauss} and \ref{prop:conc_sqrt_t_trsigma2} are smaller than their target quantities $\|\Sigma\|_{\mathrm{op}}^{\nicefrac{1}{2}}/2$ and $(\tr \Sigma^2)^{\nicefrac{1}{2}}/2$, respectively. The requirement  that the size of the sample is larger than the effective dimension $d_e$ appears mild.}

For the bounded \JB{setting} and the signal detection problem~\eqref{eq:mmodel},  estimators $\widehat{Q}_1$ and $\widehat{Q}_2$ satisfying Assumption~\ref{ass:ass2} can also be constructed
in a similar way \GB{from Propositions~\ref{prop:conc_sigma_hat_bounded} and~\ref{prop:conc_sqrt_t_trsigma2_bounded}}
(details omitted). In the bounded \JB{setting}, the quantiles $q_1$ and $q_2$ of $U$ are composed of two terms, the first (and larger) one  gives the dependence in the covariance of the distribution, the second depends
on the bound $L$. This additional term will have to be taken into account,
and the condition on $n$ analogous to~\eqref{eq:cond_n} will involve $L$.
In general this will not be a problem since $L$ or an upper bound on $L$ is supposed to be known,
as is the case for instance in the kernel setting (see the concluding discusssion in the previous section).
Finally, for the two-sample test problem~\eqref{eq:mmodel_2sample}, comparable results can be obtained using the estimators $\widehat{\Sigma}(\YY)$ and $\widehat{T}(\YY)$;
we omit the details.

\subsection{Concluding remarks}

\subsubsection{A technical discussion point: Gaussian, sub-Gaussian, and bounded vectors.}
\label{sse:discsubg}

The utility of our systematic distinction between the Gaussian and bounded case can be disputed in the light
of recent concentration literature (see e.g. \citealp{Hsu12,Kol17} and
further references therein) deriving results holding for sub-Gaussian random vectors,
a seemingly more general setting emcompassing both the Gaussian and bounded \JB{settings} as
particular cases (since bounded variables are sub-Gaussian by Hoeffding's inequality).

This point deserves a specific discussion. The sub-Gaussianity assumption for a vector variable $X$
(assumed centered for simplicity here) often takes the following form:
for any unit vector $u$, denoting $X_u = \inner{X,u}$, it is assumed that $\norm{X_u}_{\psi_2} \leq C \sqrt{\var{X_u}}$ (where
$\norm{.}_{\psi_2}$ is the Orlicz $\psi_2$-norm); or equivalently in terms of Laplace transform,
\begin{equation}
  \label{eq:subg}
  \log(\e{\exp \lambda(X_u)}) \leq (C')^2 \lambda^2 \var{X_u}/2 \text{ for all } \lambda \geq 0.
\end{equation}
A key point is that the factors $C$ or $C'$ in those
definitions should be independent of $u$, and they generally appear as global factors in the
derived deviation inequalities. If the only information we have is that $\norm{X}$ is bounded a.s.
by $L$, we see that the factors $C$ or $C'$ should be taken of the order of $\sup_{\norm{u}=1} (L/\sqrt{\var{X_u}})=
L \norm[0]{\Sigma^{-1}}_{\mathrm{op}}^{\nicefrac{1}{2}}$, which is not acceptable in a high-dimensional setting, and
in particular for the application to KME described in Section~\ref{se:kme}, where one might expect
that $\norm[0]{\Sigma^{-1}}_{\mathrm{op}}$ can get arbitrarily large or even infinite.

Some works (such as \cite{SpoZhi13} and the appendix of \cite{SpoDic15}) consider settings going beyond
sub-Gaussianity, i.e. when~\eqref{eq:subg} is only required to hold for $\lambda \leq M^{-1}$. This allows in principle
for more general variables, e.g. chi-squared type statistics or variables admitting
Berstein- or Bennett-type control of their Laplace transform, while making the constant $C'$
in \eqref{eq:subg} controlled by a fixed numerical constant.
Under this assumption the ``first-order'' terms are of the correct order, i.e. typically only depend on the variance $\Sigma$. Unfortunately, the value of $M$
comes up into additional terms, and since its value has to be independent of $u$, in the bounded
setting the uniformity with respect to $u$ means that $M$ should be again taken of the order of $\norm[0]{\Sigma^{-1}}_{\mathrm{op}}^{\nicefrac{1}{2}}$.

To summarize, despite our best efforts we were not able to derive from existing general results,
working under the (possibly extended) sub-Gaussian assumption,
a concentration in the bounded \JB{setting} that would not involve $\norm[0]{\Sigma^{-1}}_{\mathrm{op}}$, and this is
the reason why we treated it separately with tools specific to bounded variables such as the
Bousquet-Talagrand inequality. It would be of course of notable interest to obtain results
under a general sub-Gaussian assumption $\sup_{\norm{u}=1}\norm{X_u}_{\psi_2} \leq L$, and control
deviations only involving various norms of $\Sigma$ for the main terms, possibly $L$ for smaller-order
terms, but not depending on $\norm[0]{\Sigma^{-1}}_{\mathrm{op}}$.

\subsubsection{Perspectives.}
We finally list a few items for future developments.
\begin{itemize}
\item It would be interesting to obtain a version of Proposition~\eqref{prop:conc_sigma_hat_bounded} where the main term does not involve the bound $L$.
\item A recent trend of research developed ``robust'' exponential concentration bounds for
  estimators of  scalars and vectors with minimal moments assumptions (see e.g. \citealp{LugMen19} for a survey of recent advances). It seems a very interesting question
  to study if such robust procedures can be pushed to the testing setting and enjoy
  similar nonasymptotic controls to the Gaussian and bounded settings
  under much relaxed distributional assumptions. \GB{Preliminary calculations seem to indicate
    that the ``median-of-means'' (MoM) approach can be applied to U-statistics without particular problems
    and that Assumption~\ref{ass:ass1} can be granted for MoM versions of U-statistics under the assumption
    of existing moments of order 4, and presumably Assumption~\ref{ass:ass2} under moments of order 8.}
\item We have analyzed here quantile estimation by direct estimation of unknown quantities coming into the quantile bounds. In practice, quantile estimation by some form of resampling procedure
  would be often sharper and preferred. V.~Spokoiny also made notable recent contributions to this
  topic \citep{SpoZhi15,NauSpoUly19}. \GB{ In the setting of two-sample
    testing where the null hypothesis is strict equality, it is possible to
    obtain tests with exact nonasymptotic level based on permutation tests and variations
    thereof; see \citet{Fro12} for such approaches for testing equality of distributions
    based on the KME methodology, and \citet{Kim20} for recent broad results on minimax
    optimality for the power of permutation-based tests.}
 \GB{Estimating quantiles via bootstrap procedures is also an interesting direction to pursue
   in setting, in the case where the null hypothesis is based on closeness rather than equality of signals, so that exact permutation tests do not apply; \citet{DetKokAue20} recently proposed nonstandard bootstrap procedures to tackle
 this issue.}
  \item Lower bounds establishing the optimality of the separation rates appearing have been established in the Gaussian case in Theorem~\ref{thm:sigdetgauss}. It would be nice find such a lower bound in the bounded case.
\end{itemize}

\section{Proofs}
\JB{The proofs of some of the technical results, first stated without justification along the text, can be found in Section~\ref{se:supp}.
  We first state a standard technical lemma which we will use several times in the following proofs.
\begin{lemma}\label{lem:sqrt_inequality}
Let $a\in \mbr_+$ and $ b \in \mbr$, then
\begin{equation}\label{eq:sqrt_inequality}
  - \min\paren{\sqrt{b}, \frac{|b|}{a}} \leq \sqrt{(a^2+b)_+} - a \leq \min\paren{ \sqrt{|b|}, \frac{|b|}{2a} }\,.
\end{equation}
\end{lemma}
}
\subsection{Proof of Theorem \ref{thm:main}}
\label{se:prmain}
 \GB{Let us denote $ D := \|\mu - \nu\|$.} Under $(H_0)$ we have \GB{$D\leq \eta$} and thus:
\begin{align*}
  \ee{H_0}{T} &= \probb{H_0}{U - \eta^2 >  2\eta\widehat{Q}_1 + 2\widehat{Q}_2} \\
  &\leq \probb{H_0}{U > D^2  + D q_1+ q_2 } \\
   &\qquad + \probb{H_0}{  \Big| q_1 - \widehat{Q}_1 \Big| > q_1/2 } + \probb{H_0}{  \Big| q_2 - \widehat{Q}_2 \Big| > q_2/2 } \\
   &\leq 3\alpha\,,
\end{align*}
where we have used Assumptions~\ref{ass:ass1} and~\ref{ass:ass2}.
\smallbreak
\JB{
We will prove below that under $(H_1)$, we have
\begin{equation}\label{eq:type2error_step1}
\probb{H_1}{D^2 - D q_1(u) - q_2(u) \leq \eta^2 +  \eta 2\widehat{Q}_1 + 2\widehat{Q}_2} \leq 2 \alpha\,,
\end{equation}
which entails:
\begin{align*}
  \probb{H_1}{ T = 0} & = \probb{H_1}{ U- \eta^2 \leq 2\eta\widehat{Q}_1 + 2\widehat{Q}_2} \\
  &\leq  \probb{H_1}{  U   \leq D^2 - D q_1 - q_2 } \\
  & \qquad + \probb{H_1}{D^2 - D q_1 - q_2 \leq \eta^2 +  \eta 2\widehat{Q}_1 + 2\widehat{Q}_2} \\
  &\leq 3\alpha\,,
\end{align*}}
and the proof is complete.
We now prove inequality~\eqref{eq:type2error_step1}. Let us first solve the following quadratic inequality in $Z\geq 0$:
\begin{equation}\label{eq:type2error_step2}
Z^2 - Zq_1 - q_2 \geq \eta^2 +  3\eta q_1 + 3q_2\,.
\end{equation}
The equation is satisfied when
\begin{equation*}
  Z \geq  \frac{q_1 + \sqrt{(2\eta + 3q_1)^2 + 16q_2}}{2}\,;
\end{equation*}
furthermore, by Lemma~\ref{lem:sqrt_inequality} and the assumed inequality~\eqref{eq:hyp_separation_rate}, we have that
 \begin{equation*}
\frac{q_1 + \sqrt{(2\eta + 3q_1)^2 + 16q_2}}{2}\leq   \eta + 2q_1 + \min\left( 2\sqrt{q_2} , \frac{2q_2}{\eta} \right) \leq \eta + \delta \,.
  \end{equation*}
  Under $(H_1)$, $D \geq \eta + \delta $,  so $D$ satisfies equation \eqref{eq:type2error_step2}. We conclude by remarking that,
    using Assumption~\ref{ass:ass2}: 
  \begin{multline*}
    \probb{H_1}{D^2 - D q_1(u) - q_2(u) \leq \eta^2 +  \eta 2\widehat{Q}_1 + 2\widehat{Q}_2}\\
    \begin{aligned}
  &\leq \prob{ \eta^2 +  \eta 2\widehat{Q}_1 + 2\widehat{Q}_2 \geq \eta^2 +  3\eta q_1 + 3q_2 }\\
  &\leq 2\alpha\,.
  \end{aligned}
  \end{multline*}
\qed

\subsection{Proof of Propositions \ref{prop:conc_U_Gaussian_case} and \ref{prop:conc_U_bounded_case}}\label{se:U}

As much for the Gaussian case as for the bounded case, we will give concentration bounds for the statistic $U$ defined in \eqref{eq:def_U}, by decomposing the statistic in four parts. Let us define:
\begin{gather*}
  U_\XX := \frac{1}{n(n-1)} \sum_{\substack{i,j=1 \\ i \neq j}}^n \inner{ X_i - \mu, X_j - \mu }\,, \\
  U_\YY := \frac{1}{m(m-1)} \sum_{\substack{i,j=1 \\ i \neq j}}^m \inner{ Y_i - \nu, Y_j - \nu }\,, \\
  U_{\XX,\YY} :=  \frac{1}{nm} \sum_{i=1}^n \sum_{j=1}^{m} \inner{ X_i - \mu, Y_j - \nu }, \\
  \Um  := \inner{ \frac{1}{n} \sum_{i=1}^n (X_i - \mu) - \frac{1}{m} \sum_{j=1}^m (Y_j - \nu) , \mu - \nu }\,.
\end{gather*}
We have that
\begin{equation}\label{eq:U_decomposition}
  U = \| \mu - \nu \|^2 - 2\Um + U_\XX + U_\YY - 2U_{\XX, \YY}\,.
\end{equation}

\subsubsection{Gaussian \JB{setting}.}
We first need some results on Gaussian variables. The first result
is a decoupling theorem of \citet{Ver19}.
\begin{proposition}[\citealp{Ver19}, Theorem~6.1.1]\label{prop:decoupling}
Let $X_1,\ldots,X_n$ be independent centered \JB{and weakly (i.e. Pettis) integrable vectors in a Hilbert space, $(a_{ij})_{1\leq i,j\leq n}$ a family of real numbers and $F: \mbr \mapsto \mbr$ a convex function.} Then
\begin{equation*}
  \e{F\paren[4]{ \sum_{i\neq j} a_{ij} \inner{ X_i , X_j } }} \leq \e{F\paren[4]{4 \sum_{i, j} a_{ij} \inner{
   X_i , X_j' } }},
\end{equation*}
where $(X'_i)$ is an independent copy of $(X_i)$.
\end{proposition}
The following lemma is standard; see e.g. \citet{Bir01}, Lemma~8.2.
\begin{lemma}\label{lem:birge}
Let $X$ a real random variable such that for all $0 < t < b^{-1}$:
\begin{equation*}
  \log \paren{\e{ e^{tX}} } \leq \frac{(at)^2}{1-bt}\,,
\end{equation*}
where $a$ and $b$ are two positive constants. Then, for all $t \geq 0$:
\begin{equation*}
  \prob{ X \geq 2a \sqrt{t} + bt }  \leq e^{-t}\,.
\end{equation*}
\end{lemma}

\begin{proposition}\label{prop:lapl_trans_gauss_vector_inner}
  Let $X$ and $Y$ be two independent Gaussian vectors following the distributions
  $\cN( 0 , \Sigma)$ and $\cN(0, S)$ respectively. Then for $t <(\|S\|_{\mathrm{op}}\|\Sigma\|_{\mathrm{op}})^{-\nicefrac{1}{2}}$:
\begin{align*}
  \log \e{\exp\paren{ t\inner{X,Y } }}\leq \frac{t^2 \tr(S\Sigma)}{2(1 - t \sqrt{\|S\|_{\mathrm{op}}\|\Sigma\|_{\mathrm{op}}})}\,.
\end{align*}
Using Lemma~\ref{lem:birge}, for all $u \geq 0$:
\begin{equation*}
  \prob{\inner{ X,Y  } \geq \sqrt{2 \tr(S\Sigma)u} +  \sqrt{\|S\|_{\mathrm{op}}\|\Sigma\|_{\mathrm{op}}}u } \leq e^{-u} .
\end{equation*}
\end{proposition}

We can now prove Proposition~\ref{prop:conc_U_Gaussian_case}. The samples $\XX$ and $\YY$ have respective distributions $\cN( \mu , \Sigma)$ and $\cN(\nu , S)$. We will obtain a concentration inequality for $U$ using its decomposition \eqref{eq:U_decomposition}.
\smallbreak
 Let us first find concentration inequalities for $U_\XX$ and $U_\YY$. Using decoupling (see Proposition \ref{prop:decoupling}) we have for all $t <(4 \|\Sigma\|_{\mathrm{op}})^{-1} $:
\begin{equation*}
  \e{\exp\paren{tn(n-1)  U_\XX}} \leq \e{\exp\paren[3]{4t \inner[3]{
   \sum_{i=1}^n X_i-\mu , \sum_{i=1}^n X_i'-\mu } }}\,,
\end{equation*}
where $X_i'$ are independent copies of the $X_i$s. Then using Proposition \ref{prop:lapl_trans_gauss_vector_inner}, it holds with probability at least $1- 2e^{-u}$:
\begin{equation}
  n(n-1)\abs{  U_\XX  } \leq 4n\paren{\sqrt{2 \tr \Sigma^2 u} + \|\Sigma\|_{\mathrm{op}}u }\,.
\end{equation}
The same method works for $U_\YY$. The concentration of $U_{\XX,\YY}$ is directly obtained using Proposition~\ref{prop:lapl_trans_gauss_vector_inner}. Finally $\Um$
is a centered 1-dimensional Gaussian with variance $(\mu-\nu )^T\paren{ \frac{\Sigma}{n} + \frac{S}{m} }(\mu -\nu )$ and we use the classical bound $\prob{ \abs{N} \geq \sigma \sqrt{2t}} \leq 2 e^{-t}$ for $N\sim \cN(0,\sigma^2)$.
Thus we obtain that with probability at least $1 - 8 e^{-u}$:
\begin{align*}
  \abs{U - \|\mu- \nu\|^2} \leq & \frac{4}{n-1}\paren{\sqrt{2 \tr \Sigma^2 u} + \|\Sigma\|_{\mathrm{op}}u } + \frac{4}{m-1}\paren{\sqrt{2 \tr S^2 u} + \|S\|_{\mathrm{op}}u } \\
   &+ \frac{4}{\sqrt{nm}} \paren{ \sqrt{2\tr \Sigma Su} + \paren{ \|\Sigma\|_{\mathrm{op}}\| S \|_{\mathrm{op}}}^{\frac{1}{2}} u} \\
   &+ \sqrt{2(\mu-\nu )^T\paren{ \frac{\Sigma}{n} + \frac{S}{m} }(\mu -\nu )u}\,.
\end{align*}
We conclude by upper bounding the operator norms $\|\Sigma\|_{\mathrm{op}} $ and $\|S\|_{\mathrm{op}}$ by $\sqrt{\tr \Sigma^2}$ and $\sqrt{ \tr S^2}$ and for the third term we use that \[
  (2\tr(\Sigma S))^{\frac{1}{2}} \leq (4\tr \Sigma^2 \tr S^2)^{\frac{1}{4}} \leq (\tr \Sigma^2)^{\frac{1}{2}} + (\tr S^2 )^{\frac{1}{2}}.\]
We finally use $(n-1)^{-1} \leq 2n^{-1}$ for $n\geq 2$ and similarly for $m$.
It is easy to check that the fourth term is upper bounded by $q_1$ defined in~\eqref{eq:defq1}.
It just remains to use that $u \geq 1$ to get $u \geq \sqrt{u}$ and \eqref{eq:conc_U_Gaussian_case}.

\subsubsection{Bounded \JB{setting}.}

The concentration of $U$ is obtained in the bounded setting using a concentration
inequality \JB{for degenerate U-statistics} of \citet{Hou03}. We present here a somewhat simplified version suited for our purpose\footnote{\GB{In the original result the $u$ deviation term involves an additional constant $D$ and we simply use $D\leq C$ here.}}.

\begin{theorem}[\citealp{Hou03}, Theorem~3.4]\label{thm:houdre_reynaud-bouret}
   Let $T_1,\ldots,T_N$  be independent random variables on a probability space $(\Omega, \Ff, \PP)$ with values in a Borel space $(\mathcal{T}, \mathcal{G})$.
  Let
  $$U_N = \sum_{i=2}^{N} \sum_{j=1}^{i-1} g_{i,j}(T_i,T_j)\,,$$
  where $g_{i,j} : \mathcal{T} \times \mathcal{T} \to \RR$ are measurable Borelian functions satisfying
  $$\EE[g_{i,j}(T_i,T_j) |T_i ] = \EE[g_{i,j}(T_i,T_j) | T_j ] = 0\,.$$
  Let us suppose that the following quantities are finite
  \begin{gather*}
    A := \underset{t,t',i,j}{\sup}|g_{i,j}(t,t')|\,, \\
    B^2 := \max \set{ \underset{t,i}{\sup} \paren[4]{ \sum_{j=1}^{i-1} \EE[g_{i,j}(t,T_j)^2] },
    \underset{t,j}{\sup} \paren[4]{ \sum_{i= j+1}^{n}  \EE[g_{i,j}(T_i,t)^2] } }\,,\\
    C^2 := \sum_{i=2}^{N} \sum_{j=1}^{i-1} \EE[g_{i,j}(T_i,T_j)^2 ]\,.
  \end{gather*}
  Then for all $u> 0$:
  \begin{align}\label{eq_houdre_reynaud}
    \PP\Bigg[U_N \geq 4C( \sqrt{2u} + 2 \sqrt{2} u) +202Bu^{3/2} + 196 Au^2 \Bigg] \leq 2.77e^{-u}\,.
  \end{align}
\end{theorem}

Let us prove Proposition \ref{prop:conc_U_bounded_case}. We recall that we suppose here that the samples $\XX$ and $\YY$ are both bounded by $L$. To obtain a deviation inequality for the statistic $U$, we consider separately
the statistics $U_\XX + U_\YY -2U_{\XX,\YY}$ and then $\Um$.
\smallbreak
Using
Theorem \ref{thm:houdre_reynaud-bouret} with $N = n+m $, $T_i := X_i - \mu $ for $1\leq i \leq n$ and $T_i = Y_i - \nu$ for $n+1 \leq i \leq n+m$, $\cT = \{ u : \|u\|\leq 4L^2 \}$  and
  \begin{equation*}
    g_{ij}(\cdot,\cdot) = \begin{cases}
                       \frac{1}{n(n-1)} \inner{ \cdot , \cdot } , & \text{if $1 \leq i,j \leq n$,} \\
                       \frac{1}{m(m-1)} \inner{ \cdot , \cdot } , & \text{if $n+1 \leq i,j \leq n+m$, } \\
                       -\frac{1}{nm} \inner{ \cdot , \cdot }, & \text{otherwise,}
                     \end{cases} 
  \end{equation*}
   we get that with probability greater than $1- 5.54e^{-u}$:
  \begin{equation}\label{eq:conc_U1}
    \abs{ U_\XX + U_\YY - 2U_{\XX,\YY} }/2 \leq 307 \paren[3]{ \frac{\sqrt{\tr \Sigma^2}}{n} + \frac{\sqrt{\tr S^2}}{m}}u  + 1854 L^2u^2\,.
  \end{equation}
    To obtain the above, we have upper bounded $A,B,C$ by:
  \begin{gather*}
  A \leq \frac{8L^2}{(n\wedge m)^2}\,, \quad B^2 \leq \frac{8L^2}{(n\wedge m)^2} \paren{ \frac{\|\Sigma\|_{\mathrm{op}}}{n} + \frac{\|S\|_{\mathrm{op}}}{m}}\,,  \\
  C^2  = \frac{3}{2}\paren{ \frac{\tr\Sigma^2}{n} + \frac{\tr S^2}{m}}\,; 
\end{gather*}
then, using that $2\sqrt{ab} \leq a+ b$ and that $\|\Sigma\|_{\mathrm{op}} \leq \sqrt{\tr \Sigma^2}$,
we get \eqref{eq:conc_U1}.

For $\Um$, we use Bernstein's inequality (i.e. combining Lemmas~\ref{lem:lap_trans_bounded_sum} and~\ref{lem:birge}) to get that with probability at least $1- 2e^{-u}$,
it holds:
\begin{equation}\label{eq:conc_U_bernstein_term}
  \abs{ \Um} \leq \|\mu - \nu \| \paren{ \sqrt{2 \paren{\frac{\|\Sigma\|_{\mathrm{op}}}{n} + \frac{\|S\|_{\mathrm{op}}}{m} } u } + \frac{2Lu}{3n \wedge m} }\,.
\end{equation}
\GB{Combining~\eqref{eq:conc_U1} and~\eqref{eq:conc_U_bernstein_term}, we obtain the claim of Proposition~\ref{prop:conc_U_bounded_case}.}
\qed

\subsection{Proof of Theorem \ref{thm:sigdetgauss}}

The upper bound is directly obtained using Theorem~\ref{thm:main}. Assumption~\ref{ass:ass1} is satisfied as
a consequence of Proposition~\ref{prop:conc_U_Gaussian_case}.
\GB{We do not consider estimation of nuisance parameters
related to the covariance matrix $\Sigma$ which is assumed to be fixed and known for this result; thus
Assumption~\ref{ass:ass2} is trivially satisfied by taking $\wh{Q}_1=q_1(\Sigma,\alpha)$, $\wh{Q}_2=q_2(\Sigma,\alpha)$.}

Let us now prove the lower bound~\eqref{eq:signalsep_low}. The following proof is an adaptation to the non-isotropic Gaussian setting of the proof of Theorem~5.1 in \citet{BlaCarGut18}.
Let $\alpha \in (0,1)$,
ans $\Sigma$ be a \JB{positive semidefinite} matrix.
Without loss of generality, we can assume that $\Sigma$ is diagonal: $\Sigma = \mathrm{diag}(\lambda_1,\ldots,\lambda_d)$ with $\lambda_1 \geq \ldots \geq \lambda_d>0$. Let us denote $\mbp_{\mu,\Sigma}$ the distribution of $\cN(\mu, \Sigma)$ for $\mu\in \mbr^d$ and introduce the Gaussian mixture distribution:
\begin{equation}\label{eq:def_qn}
 \mbq_\Sigma^n := \frac{1}{2^{d-1}} \sum_{m \in \cM} \mbp_{m,\Sigma}^{\otimes n},
\end{equation}
where
\begin{equation*}
   \cM = \set{ \paren{  \lambda_1 v_1 h,\ldots,  \lambda_{d-1}v_{d-1} h, \eta } | \, v \in \{-1,1\}^{d-1} }\,.
\end{equation*}
We take $h^2 := \frac{(\eta+\delta)^2- \eta^2}{\tr \Sigma^2 - \lambda_d^2}$. Then, for all $m \in \cM$,
\begin{equation*}
\|m\|_d = \sqrt{\eta^2 + (\tr \Sigma^2 - \lambda_d^2)h^2} = \eta+\delta\,.
\end{equation*}

Let $\nu = (0,\ldots,\eta)$, it holds
\begin{align}
  \underset{\mbp \in \cH_0}{\sup~} \mbp^{\otimes n}\paren{ \phi =1 } + \underset{\mbp \in \cA_\delta}{\sup } \mbp^{\otimes n} \paren{ \phi =0 } & \geq \mbp_{\nu,\Sigma}^{\otimes n}\paren{\phi =1} + \mbq_\Sigma^n\paren{\phi = 0} \notag \\
  & \geq 1 - \frac{1}{2} \norm{ \mbp_{\nu,\Sigma}^{\otimes n} - \mbq_\Sigma^n }_{\mathrm{TV}} \notag \\
  & \geq 1 - \frac{1}{2} \paren{ \int_{\mbr^{d\times n}} \paren[4]{\frac{d \mbq_\Sigma^n}{d \mbp_{\nu,\Sigma}^{\otimes n}}}^2 d\mbp_{\nu,\Sigma}^{\otimes n} -1}^{\frac{1}{2}}\,. \label{eq:lowchideux}
\end{align}
see for instance \citet{Bar02}. For a tensor product of Gaussian distributions with fixed, equal covariance,
the empirical mean is a sufficient statistic \JB{because the Radon-Nikodym derivative of a tensor product of Gaussian measures w.r.t. the Lebesgue measure can be written for $x_1,\ldots,x_n \in \mbr^d$ as
\begin{equation*}
  \frac{d \mbp^{\otimes n}_{m,\Sigma}}{d \lambda^{\otimes n}}(x_1,\ldots,x_n) = \phi_{m,\Sigma/n}(\bar{x})F_\Sigma(x_1,\ldots,x_n)\,,
\end{equation*}
where $\bar{x}$ is the mean of the $x_i$s,} \GB{$\phi_{m,\Sigma/n}$ is the p.d.f. of a normal $\cN(m,{\Sigma}/{n})$ variable, and $F_\Sigma$ is a function of $(x_1,\ldots,x_n)$ which only depends on $\Sigma$. Therefore
\begin{equation*}
  {\frac{d \mbq_\Sigma^n}{d \mbp_{\nu,\Sigma}^{\otimes n}}}(x_1,\ldots,x_n) =
  \frac{d \mbq^1_{\Sigma/n}}{d \mbp_{\nu,\Sigma/n}}(\bar{x}),
\end{equation*}
and thus
\begin{equation*}
   \int_{\mbr^{d\times n}} \paren[3]{\frac{d \mbq_\Sigma^n}{d \mbp_{\nu,\Sigma}^{\otimes n}}}^2  d\mbp_{\nu,\Sigma}^{\otimes n} = \int_{\mbr^d} \paren[3]{\frac{d \mbq^1_{\Sigma/n}}{d \mbp_{\nu,\Sigma/n}}}^2  d\mbp_{\nu,\Sigma/n}\,.
\end{equation*}}
Thus the problem boils down to studying a single Gaussian vector of  covariance ${\Sigma}/{n}$; for the following we will assume $n=1$ and replace at the end $\Sigma$ by ${\Sigma}/{n}$.
Let us compute the densities $F_\nu$ and $Q$ of these two distributions. For $x\in \mbr^d$:
\begin{equation*}
  F_\nu(x) = \paren{ \det \Sigma (2\pi)^d }^{-\frac{1}{2}} \exp\paren{ - \frac{1}{2\lambda_d}(x_d - \eta)^2} \prod_{i=1}^{d-1}  \exp\paren{ - \frac{x_i^2}{2\lambda_i}}\,,
\end{equation*}
and
\begin{align*}
  Q(x)  &= \paren{ \det \Sigma (2\pi)^d }^{-\frac{1}{2}} \exp\paren{ - \frac{1}{2\lambda_d}(x_d - \eta)^2}\\
  & \qquad \qquad  \times \frac{1}{2^{d-1}} \sum_{\substack{v_i \in \{-1,1\} \\ 1\leq i \leq d-1 }} \prod_{i=1}^{d-1}  \exp\paren{ - \frac{1}{2\lambda_i}(x_i-h\lambda_iv_i)^2} \\
        & = \paren{ \det \Sigma (2\pi)^d }^{-\frac{1}{2}} \exp\paren[3]{ - \frac{1}{2\lambda_d}(x_d - \eta)^2 - \frac{h^2}{2} \sum_{i=1}^{d-1}\lambda_i}\\
  & \qquad \qquad \times \prod_{i=1}^{d-1}  \exp\paren{ - \frac{x_i^2}{2\lambda_i}} \cosh \paren{ h x_i } \,.
\end{align*}
Using that $\e{ \JB{\cosh^2(aZ)} } = \exp(a^2\sigma^2) \cosh(a^2\sigma^2) $ when $Z \sim \cN(0, \sigma^2)$, we have that
\begin{align*}
  \int_{\mbr^d} \frac{Q(x)^2}{F_\nu(x)}dx & = \paren{ \det \Sigma (2\pi)^d }^{-\frac{1}{2}} \exp\paren[3]{- h^2 \sum_{i=1}^{d-1}\lambda_i} \int_{\mbr} \exp\paren{ -\frac{1}{2\lambda_d} (x_d - \eta)^2}dx_d   \\
  & \qquad \times \prod_{i=1}^{d-1} \int_{\mbr} \cosh^2\paren{ hx_i} \exp \paren{ -\frac{x_i^2}{2\lambda_i} }dx_i \\
  &=  \exp\paren[3]{- h^2 \sum_{i=1}^{d-1}\lambda_i} \prod_{i=1}^{d-1} \exp(h^2\lambda_i) \cosh(h^2\lambda_i) \\
  &= \prod_{i=1}^{d-1} \cosh(h^2\lambda_i).
  \end{align*}
By Taylor expansion, we obtain the bound
\begin{equation*}
  h^2\lambda_i \leq 1 \Rightarrow \cosh(h^2\lambda_i) \leq 1 + \frac{e}{2}\lambda_i^2h^4\,.
\end{equation*}
From this and the definition of $h$ we deduce:
\begin{equation*}
 \log \prod_{i=1}^{d-1} \cosh(h^2\lambda_i) \leq \frac{e}{2} ( \tr \Sigma^2-\lambda_d^2) h^4 = \frac{e}{2(\tr \Sigma^2 - \lambda_d^2)} \paren{(\eta +\delta)^2 - \eta^2 }^2\,.
\end{equation*}
The end of the proof follows the same steps as the proof of Theorem 5.1 of \cite{BlaCarGut18}.
That leads us to the final result: if
\begin{equation*}
  \delta \leq \sqrt{\|\Sigma\|_{\mathrm{op}} \sqrt{d_* -1} s + \eta ^2} - \eta \quad \text{where} \quad s := \sqrt{\frac{2}{e} \log(1+ 4(1-\alpha)^2)}\,,
\end{equation*}
and
\begin{equation*}
  d_* \geq 1 + \frac{2}{e} \ln (5), \quad \text{i.e.} \quad d_* \geq 3\,,
\end{equation*}
then using~\eqref{eq:lowchideux}
\begin{equation*}
   \underset{\mbp \in \cH_0}{\sup } \, \prob{ \phi =1 } + \underset{\mbp \in \cA_\delta}{\sup } \prob{ \phi =0 } > \alpha\,.
\end{equation*}
It follows
\begin{align*}
  \delta^* & \geq \sqrt{\|\Sigma\|_{\mathrm{op}} \sqrt{d_* -1} s + \eta ^2} - \eta\\
  &\geq 2^{-\frac{3}{2}} \min \paren{ \sqrt{s\|\Sigma\|_{\mathrm{op}}}(d_*-1)^{\frac{1}{4}}, s\|\Sigma\|_{\mathrm{op}} \frac{(d_*-1)^{\frac{1}{2}}}{\eta} },
\end{align*}
and we obtain \GB{the inequality corresponding to the second part of the maximum in the right-hand side of}
\eqref{eq:signalsep_low} by using that $s \geq (1- \alpha)$ and that $d_*-1 \geq 2d_*/3$ because $d_*\geq 3$.

Let us prove now that $\delta^* \gtrsim \sqrt{\|\Sigma\|_{\mathrm{op}}}$. Let us consider the eigenvector $e_1$ associated to the maximum eigenvalue $\|\Sigma\|_{\mathrm{op}}$. Then $\mbp_{(\eta+\delta)e_1,\Sigma} \in \cH_0$ and $\mbp_{(\eta+\delta)e_1,\Sigma}\in \cA_\delta$. Let us denote $\lambda_1 = {\|\Sigma\|_{\mathrm{op}}}/{n}$, we have:
\begin{align*}
  \int_{\mbr^d}\paren[3]{ \frac{d\mbp^{\otimes n}_{(\eta+\delta)e_1,\Sigma}}{d\mbp^{\otimes n}_{\eta e_1,\Sigma}}}^2 d\mbp^{\otimes n}_{\eta e_1,\Sigma}&=  \int_{\RR^d}\paren{ \frac{d\mbp_{(\eta+\delta)e_1,\Sigma/n}}{d\mbp_{\eta e_1,\Sigma/n}}}^2 d\mbp_{\eta e_1,\Sigma/n} \\
  &= \frac{e^{-\delta^2/\lambda_1}}{\sqrt{\lambda_1 2\pi}}\int_\mbr \exp\paren{ -\frac{(x-\eta)^2}{2\lambda_1}} \exp \paren{ \frac{2\delta(x-\eta)}{\lambda_1}} dx \\
  &= \exp \paren{ \frac{3\delta^2}{\lambda_1} }\,.
\end{align*}
If $\delta \leq \sqrt{\frac{\lambda_1}{3} \log\paren{ 1 + 4(1-\alpha)^2}}$, then using~\eqref{eq:lowchideux}
\begin{equation*}
   \underset{\mbp \in \cH_0}{\sup } \, \prob{ \phi =1 } + \underset{\mbp \in \cA_\delta}{\sup } \prob{ \phi =0 } > \alpha\,.
\end{equation*}
It follows that:
\begin{equation*}
  \delta^* \geq \sqrt{\|\Sigma/n\|_{\mathrm{op}}(1-\alpha)}.
\end{equation*}

\subsection{Proof of Theorem \ref{thm:sigdetgauss_2sample}} \label{se:proofth3}

\JB{This proof is similar to the proof of Theorem \ref{thm:sigdetgauss}, so some details will be skipped. As in the one-sample case the upper bound is directly obtained using Theorem~\ref{thm:main} and Proposition~\ref{prop:conc_U_Gaussian_case}. We just additionally use the following upper bounds:
\begin{align*}
  \frac{\sqrt{\tr \Sigma^2}}{n}+ \frac{\sqrt{\tr S^2}}{m}
  &\leq \sqrt{2}\sqrt{\frac{\tr \Sigma^2}{n^2} + \frac{\tr S^2}{m^2}} \leq \sqrt{2}\sqrt{\tr\paren{ \frac{\Sigma}{n}+ \frac{S}{m}}^2}\,;\\
\frac{\|\Sigma\|_{\mathrm{op}}}{n} + \frac{\|S\|_{\mathrm{op}}}{m}
  & \leq 2 \max\paren{\frac{\|\Sigma\|_{\mathrm{op}}}{n} , \frac{\|S\|_{\mathrm{op}}}{m}}
  \leq 2 \norm{\frac{\Sigma}{n} + \frac{S}{m}}_{\mathrm{op}},
\end{align*}
where the last inequality holds because $\Sigma, S$ are both positive semidefinite.

\GB{The lower bound in the two-sample case is a direct consequence of the one-sample case,
by reduction to the case where one of the two sample means is known, say equal to zero.}
More specifically, let $\Sigma$ and $S$ be two symmetric positive semidefinite matrices, we consider again the distribution $\mbq^n_\Sigma$ defined in \eqref{eq:def_qn}. Then
\begin{align*}
 \int_{\mbr^{d\times (n+m)}} \paren[3]{\frac{d \mbq^n_\Sigma \otimes \mbp_{0,S}^{\otimes m}}{d \mbp_{\nu,\Sigma}^{\otimes n}\otimes \mbp_{0,S}^{\otimes m} }}^2 d\mbp_{\nu,\Sigma}^{\otimes n}\otimes \mbp_{0,S}^{\otimes m}= \int_{\mbr^{d\times n}} \paren[3]{\frac{d \mbq_\Sigma^{n}}{d \mbp_{\nu,\Sigma}^{\otimes n} }}^2 d\mbp_{\nu,\Sigma}^{\otimes n}\,.
\end{align*}
Then using the previous results of the proof of Theorem~\ref{thm:sigdetgauss} we obtain that
\begin{equation}\label{eq:delta_sig}
  \delta^*(\alpha) \geq \paren{ n^{-1}\sqrt{\tr \Sigma^2 -\lambda^2_d} s + \eta ^2}^{\frac{1}{2}} - \eta,
\end{equation}
with $s = \sqrt{\frac{2}{e} \log(1+ 4(1-\alpha)^2)}$. By the same token we obtain that
\begin{equation}\label{eq:delta_S}
  \delta^*(\alpha) \geq \paren{ m^{-1}\sqrt{\tr S^2 -\ell^2_d} s + \eta ^2}^{\frac{1}{2}} - \eta,
\end{equation}
where $\ell_d$ is the smallest eigenvalue of the matrix $S$. 
Because $d_* \geq 3$, it holds
\begin{align*}
  \max\paren{n^{-2}\paren{\tr \Sigma^2 -\lambda^2_d},m^{-2}\paren{\tr S^2 -\ell^2_d}}
  & \geq \frac{2}{3} \max\paren{n^{-2}\tr \Sigma^2 ,m^{-2}\tr S^2}^{\frac{1}{2}}\\
  &\geq \frac{1}{6} \tr \paren{ \frac{\Sigma}{n} + \frac{S}{m} }^2,
\end{align*}
and by combining \eqref{eq:delta_sig} and \eqref{eq:delta_S}, we obtain that
\begin{equation*}
  \delta^*(\alpha) \geq (2 \sqrt{12})^{-1} \sigma \min \paren{ \sqrt{s}d_*^{\frac{1}{4}}, s \frac{\sigma d_*^{\frac{1}{2}}}{\eta} }\,,
\end{equation*}
where $\sigma = \| {\Sigma}/{n} + {S}/{m} \|_\mathrm{op}$. We obtain \eqref{eq:signalsep_low_2sample} using again that $s \geq 1 -\alpha$.

The last part of the lower bound is obtained as in the one-sample case using first the distributions $\mbp_{(\eta+\delta)e_1,\Sigma}^{\otimes n} \otimes \mbp_{0,S}^{\otimes m}$ and $\mbp_{\eta e_1,\Sigma}^{\otimes n} \otimes \mbp_{0,S}^{\otimes m}$ where $e_1$ is still the eigenvector associated to the biggest eigenvalue of $\Sigma$. We obtain that $\delta^*(\alpha) \gtrsim \|{\Sigma}/{n}\|_{\mathrm{op}}^{\nicefrac{1}{2}}$. By the same token, we obtain that $\delta^*(\alpha) \gtrsim \|{S}/{m}\|_{\mathrm{op}}^{\nicefrac{1}{2}}$ and conclude the proof using that $2\max(\|{\Sigma}/{n}\|_{\mathrm{op}},\|{S}/{m}\|_{\mathrm{op}}) \geq \|{\Sigma}/{n} +{S}/{m}\|_\mathrm{op}$.
}

\subsection{Proof of Propositions \ref{prop:conc_sigma_hat_gauss} and \ref{prop:conc_sigma_hat_bounded}}\label{se:sig_hat}

We want to obtain a concentration inequality for the estimator $\sqrt{\|\widehat{\Sigma}\|_{\mathrm{op}}}$.
To this end, we will first study the following:
\begin{equation}\label{eq:def_sigma_tilde}
  \wt{\Sigma} := \frac{1}{n} \sum_{i=1}^{n} (X_i - \mu ) (X_i - \mu )^T\,,
\end{equation}
where $\mu$ is the true mean of the sample $\XX$.
\GB{Then we have:}
\begin{equation}
  \norm[1]{\widehat{\Sigma}- \wt{\Sigma} }_{\mathrm{op}}
  = \norm{-  (\mu - \widehat{\mu})(\mu - \widehat{\mu})^T}_{\mathrm{op}}
  = \|\mu - \widehat{\mu} \|^2\,.
    \label{eq:shatstilde}
\end{equation}

\subsubsection{Gaussian \JB{setting}.}

The concentration of ${\|\wt{\Sigma}\|^{\nicefrac{1}{2}}_{\mathrm{op}}}$ is a consequence of the classical
Lipschitz Gaussian concentration property (see e.g. Theorem 3.4 in~\citealp{Mas03}).
\begin{theorem}[Gaussian Lipschitz concentration] \label{thm:Gaussian_lipschitz_conc}
  Let $X = (x_1,\ldots,x_d)$ be a vector of i.i.d. standard Gaussian variables, and $f: \mbr^d \mapsto \mbr$ be a $L$-Lipschitz function with respect to the Euclidean norm. Then for all $t \geq 0$:
  \begin{equation}
    \prob{ f(X) - \e{f(X)} \geq t } \leq e^{- \frac{t^2}{2L^2}}.
  \end{equation}
\end{theorem}
The following corollary is a direct consequence of that theorem (we provide a proof in Section~\ref{se:supp}),
which will be used to control the term in~\eqref{eq:shatstilde}.
\begin{corollary}\label{cor:conc_gauss_classic}
Let $X $ a random Gaussian vector of distribution $\cN(\mu, \Sigma)$. Then for all $u\geq 0$:
\begin{equation}\label{eq:conc_gauss_norm2}
  \prob{\|X\| \geq \sqrt{\|\mu\|^2 + \tr \Sigma} + \sqrt{2 \|\Sigma\|_{\mathrm{op}}u}} \leq e^{-u}\,.
\end{equation}
\end{corollary}
We will use the results of \citet{Kol17} giving an upper bound of the expectation of the operator norm of the deviations
of $\wt{\Sigma}$ from its expectation. The constants come from the improved version given by \citet{Han17}.
\begin{theorem}[\citealp{Han17}]\label{thm:van_handel}
Let $\XX = (X_i)_{1\leq i \leq n}$ a sample of independent Gaussian vectors of distribution $\cN(0,\Sigma)$, then
\begin{equation}\label{eq:estimator_sigma_op_centered}
  \e{\norm[1]{ \wt{\Sigma} - \Sigma}_{\mathrm{op}}  } \leq \|\Sigma\|_{\mathrm{op}}\paren[3]{  (2+ \sqrt{2} ) \sqrt{\frac{d_e}{n}}+ 2 \frac{d_e}{n}}\,,
\end{equation}
where $d_e = \tr \Sigma/ \|\Sigma\|_{\mathrm{op}}$ and $\wt{\Sigma}$ is defined in equation \eqref{eq:def_sigma_tilde}.
\end{theorem}
We can now prove a concentration inequality for ${\|\wt{\Sigma}\|^{\nicefrac{1}{2}}_{\mathrm{op}}}$.

\begin{proposition}\label{prop:conc_sigma_tilde_gauss}
Let $\XX = (X_i)_{1\leq i \leq n}$ a sample of independent  $\cN(\mu, \Sigma)$ Gaussian vectors, then for $ u \geq 0$, with probability at least $1-2e^{-u}$:
\begin{equation}\label{eq:conc_sigma_tilde}
  \abs{ {\norm[1]{\wt{\Sigma}}^{\frac{1}{2}}_{\mathrm{op}}} - {\norm{\Sigma}^{\frac{1}{2}}_{\mathrm{op}}} } \leq 2\sqrt{\frac{2 \tr \Sigma}{n}}+ \sqrt{ \frac{2u\|\Sigma\|_{\mathrm{op}}}{n} }  \,,
\end{equation}
where $\wt{\Sigma}$ is defined in \eqref{eq:def_sigma_tilde}.
\end{proposition}

\begin{remark}

   In \eqref{eq:conc_sigma_tilde}, the lower and upper bounds have been brought together, but the lower bound is in fact slightly better than the upper bound. This is due to the lower bound of the expectation where $\tr \Sigma$ can be replaced by $\|\Sigma\|_{\mathrm{op}}$, see \eqref{eq:low_bound_expect_sqrt} below.

\end{remark}

\begin{proof}
   We remark that
\begin{align*}
  {\norm[1]{ \wt{\Sigma}}^{\frac{1}{2}}_{\mathrm{op}}}
  &= \sup_{\|u\|_d=1} \sqrt{u^t \wt{\Sigma}u} \\
  &= \sup_{\|u\|_d=1} \frac{1}{\sqrt{n}} \paren{\sum_{i=1}^{n} \inner{u,X_i-\mu}^2}^{\frac{1}{2}}\\
  &= \sup_{\|u\|_d=1} \sup_{\|v\|_n=1} \frac{1}{\sqrt{n}} \sum_{i=1}^{n} \inner{
   u , X_i - \mu } v_i \\
  & \overset{\mathrm{dist}}{\sim}   \sup_{\|u\|_d=1} \sup_{\|v\|_n=1} \frac{1}{\sqrt{n}} \sum_{i=1}^{n} \inner[1]{ u ,\Sigma^{\frac{1}{2}} g_i } v_i\,,
\end{align*}
where $(g_i)_{i=1\ldots n}$ are i.i.d. standard Gaussian vectors and $\|\cdot \|_p$ for $p\in \mbn$ is defined as the Euclidean norm in $\mbr^p$. Let $u$ and $v$ be unit vectors in $\mbr^d$ and $\mbr^n$ respectively and $f_{u,v}: \mbr^{d \times n } \to \mbr$:
\begin{equation*}
  f_{u,v}(y) := \frac{1}{\sqrt{n}} \sum_{i=1}^{n} \inner[1]{ u ,\Sigma^{\frac{1}{2}} y_i } v_i \,, \quad y\in \mbr^{d\times n}\,.
\end{equation*}
 These functions are Lipschitz: indeed for all $z,y \in \mbr^{d \times n }$ we have:
\begin{align}
f_{u,v}(y) - f_{u,v}(z) & = \frac{1}{\sqrt{n}} \sum_{i=1}^{n} \inner[1]{ u ,\Sigma^{\frac{1}{2}} (y_i-z_i) } v_i    \leq  \frac{1}{\sqrt{n}} \sum_{i=1}^{n} \|\Sigma\|_{\mathrm{op}}^{\frac{1}{2}} \|y_i-z_i\|_d  \abs{v_i} \notag \\
  & \leq \frac{\|\Sigma\|^{\frac{1}{2}}_{\mathrm{op}}}{\sqrt{n}} \sqrt{ \sum_{i=1}^{n} \|y_i - z_i\|^2_d }= \frac{\|\Sigma\|^{\frac{1}{2}}_{\mathrm{op}}}{\sqrt{n}} \|y-z\|_{d \times n}\,. \label{eq:lipfuv}
\end{align}
A supremum of Lipschitz functions is Lipschitz, thus we can use the Gaussian Lipschitz concentration
(Theorem \ref{thm:Gaussian_lipschitz_conc}), and get for all $x \geq 0$:
\begin{equation}\label{eq:conc_sigma_tilde_expectation}
  \prob{ { {\norm[1]{\wt{\Sigma}}^{\frac{1}{2}}_{\mathrm{op}}} - \e{{\norm[1]{\wt{\Sigma}}_{\mathrm{op}}^{\frac{1}{2}}}}} \geq \sqrt{ \frac{2x\norm{\Sigma}_{\mathrm{op}}}{n} } } \leq e^{-x}\,,
\end{equation}
with the same control for lower deviations.

It remains to upper bound $\abs{ \e{{\|\wt{\Sigma}\|_{\mathrm{op}}^{\nicefrac{1}{2}}} } - {\|\Sigma\|_{\mathrm{op}}^{\nicefrac{1}{2}}}}$. \JB{For one direction, using Jensen's and triangle inequalities  and inequality~\eqref{eq:sqrt_inequality}, we get:}
\begin{align*}
  \e{{\norm[1]{\wt{\Sigma}}^{\frac{1}{2}}_{\mathrm{op}}} } - {\|\Sigma\|^{\frac{1}{2}}_{\mathrm{op}}} & \leq \sqrt{\|\Sigma\|_{\mathrm{op}} + \e{\norm[1]{  \wt{ \Sigma} - \Sigma }_{\mathrm{op}}}}- \sqrt{\|\Sigma\|_{\mathrm{op}}} \\
                                                                              &\leq \min\paren{ \sqrt{\e{\norm[1]{  \wt{ \Sigma} - \Sigma }_{\mathrm{op}}}}, \frac{\e{\norm[1]{  \wt{ \Sigma} - \Sigma}_{\mathrm{op}}}}{2\sqrt{\|\Sigma\|_{\mathrm{op}}}}} \\
                                                                              &\leq 2\sqrt{\frac{2 \tr \Sigma}{n}}\,.
\end{align*}
\JB{For the last inequality, we have used Theorem~\ref{thm:van_handel} for the expectation, then the fact that $\min\paren{\paren{a\sqrt{x}+ bx}^{\nicefrac{1}{2}}, (a\sqrt{x}+bx)/2}\leq \max(\sqrt{a+b},(a+b)/2)\sqrt{x}$ where $a = 2+ \sqrt{2}$, $b =2$ and $x=d_e/n$. This is achieved by treating cases $x\leq 1$ and $x \geq 1$ separately.}

\GB{For the other direction, a reformulation of~\eqref{eq:conc_sigma_tilde_expectation} is that there exists a random variable
  $g\sim \mathrm{Exp}(1)$ such that:}
\begin{equation*}
{\norm[1]{\wt{\Sigma}}^{\frac{1}{2}}_{\mathrm{op}}} \leq \e{{\norm[1]{\wt{\Sigma}}^{\frac{1}{2}}_{\mathrm{op}}}} + \sqrt{ \frac{2g\|\Sigma\|_{\mathrm{op}}}{n} }\,.
\end{equation*}
Taking the square then the expectation and \JB{then applying Jensen's inequality to the concave function $x \mapsto (a+b\sqrt{x})^2$ ($a,b \geq 0$)}, we obtain:
\begin{align*}
  \|\Sigma\|_{\mathrm{op}} \leq \e{\norm[1]{\wt{\Sigma}}_{\mathrm{op}}} &\leq \ee{g}{ \paren{ \e{{\norm[1]{\wt{\Sigma}}^{\frac{1}{2}}_{\mathrm{op}}}} + \sqrt{ \frac{2g\|\Sigma\|_{\mathrm{op}}}{n} } }^2} \\
  &\leq  \paren{ \e{{\norm[1]{\wt{\Sigma}}^{\frac{1}{2}}_{\mathrm{op}}}} + \sqrt{ \frac{2\|\Sigma\|_{\mathrm{op}}}{n} } }^2 \,,
\end{align*}
and thus
\begin{equation} \label{eq:low_bound_expect_sqrt}
  \e{{\norm[1]{\wt{\Sigma}}^{\frac{1}{2}}_{\mathrm{op}}}} - {\|\Sigma\|^{\frac{1}{2}}_{\mathrm{op}} } \geq - \sqrt{ \frac{2\|\Sigma\|_{\mathrm{op}}}{n} } \geq - 2\sqrt{\frac{2 \tr \Sigma}{n}}\,.
\end{equation}
\end{proof}

\textbf{Proof of  Proposition~\ref{prop:conc_sigma_hat_gauss}.}
It holds
  \begin{align*}
    \abs{  {\norm[1]{\widehat{\Sigma}}^{\frac{1}{2}}_{\mathrm{op}}} - {\norm{\Sigma}^{\frac{1}{2}}_{\mathrm{op}}} } & \leq  \abs{ {\norm[1]{\wt{\Sigma}}^{\frac{1}{2}}_{\mathrm{op}}} - {\norm{\Sigma}^{\frac{1}{2}}_{\mathrm{op}}} } + \norm[1]{ \widehat{\Sigma} - \wt{\Sigma}}_\mathrm{op}^{\frac{1}{2}}\,.
  \end{align*}
  Then, from~\eqref{eq:shatstilde}:
  \begin{equation*}
    {\| \widehat{\Sigma} - \wt{\Sigma}\|_{\mathrm{op}}^{\frac{1}{2}}} \leq  \|\mu - \widehat{\mu} \|\,.
  \end{equation*}
  According to Proposition~\ref{prop:conc_sigma_tilde_gauss} and Corollary~\ref{cor:conc_gauss_classic}, we obtain that for $ u \geq 0$, with probability at least $1- 3 e^{-u}$:
  \begin{align*}
    \abs{  {\norm[1]{\widehat{\Sigma}}^{\frac{1}{2}}_{\mathrm{op}}} - {\norm{\Sigma}^{\frac{1}{2}}_{\mathrm{op}}} } \leq 2\sqrt{\frac{2 \tr \Sigma}{n}}+ \sqrt{ \frac{2u\|\Sigma\|_{\mathrm{op}}}{n} }  + \sqrt{\frac{ \tr \Sigma}{n}} + \sqrt{\frac{2\|\Sigma\|_{\mathrm{op}}u}{n}} \,.
    \end{align*}
    So, for $u\geq 0$, with probability at least $1- 3 e^{-u}$:
    \begin{equation*}
     \abs{  {\norm[1]{\widehat{\Sigma}}^{\frac{1}{2}}_{\mathrm{op}}} - {\norm{\Sigma}^{\frac{1}{2}}_{\mathrm{op}}} } \leq 3\sqrt{\frac{2 \tr \Sigma}{n}} + 2 \sqrt{ \frac{2u\|\Sigma\|_{\mathrm{op}}}{n} }\,.
     \end{equation*}

\subsubsection{Bounded setting.}
We first recall the following concentration result for bounded random vectors in the formulation of
\citet{Bou02}.
\begin{theorem}[Talagrand-Bousquet inequality] \label{thm:bousquet}
Assume $(X_i)_{1\leq i \leq n}$ are i.i.d. with marginal distribution $\mbp$. Let $\mathcal{F}$ be a countable set of functions from $\mathcal{X}$ to $\mbr$ and assume that all functions $f$ in $\mathcal{F}$ are $\mbp$-measurable, square-integrable, bounded by $M$ and satisfy $\e{f}=0$. Then we denote
\begin{equation*}
  Z = \underset{f \in \mathcal{F}}{\sup } \sum_{i=1}^{n} f(X_i).
\end{equation*}
Let $\sigma$ be a positive real number such that $\sigma^2 \geq \sup_{f \in \mathcal{F}} \var{f(X_1)}$. Then for all $u \geq 0$, $\varepsilon >0$ we have:
\begin{equation*}
  \prob{Z \geq \e{Z}(1+ \varepsilon) + \sqrt{2 u n \sigma^2} +\frac{Mu}{3}(1+ \varepsilon^{-1})} \leq e^{-u}.
\end{equation*}
\end{theorem}

The following corollary is a direct consequence of Theorem~\ref{thm:bousquet}. Some refinement of this result in the same vein (including two-sided deviation
control in the uncentered case) can be found in \citet{MarBlaFer20} (Proposition~6.2 and Corollary~6.3).
\begin{corollary}\label{cor:conc_norm_bounded}
Let $X_i $ for $i=1,\ldots,n$ i.i.d. random vectors bounded by $L$ with expectation $\mu$, covariance $\Sigma$ in a separable Hilbert space $\cH$. Then for $u \geq 0$, with probability at least $1-e^{-u}$:
\begin{equation*}
  \norm{\frac{1}{n} \sum_{i=1}^{n} X_i - \mu } \leq 2\sqrt{\frac{ \tr \Sigma}{n}} + \sqrt{\frac{2 \|\Sigma\|_{\mathrm{op}} u}{n} } + \frac{4L u}{3n}\,.
\end{equation*}
\end{corollary}

\begin{lemma}\label{lem:upper_bound_expectation_sigma_tilde}
Let $X_i $ for $i=1,\ldots,n$ i.i.d. random vectors bounded by $L$ with expectation $\mu$, covariance $\Sigma$ in a separable Hilbert space $\cH$. Then
\begin{equation}\label{eq:upper_bound_expectation_sigma_tilde}
  \e{\norm[1]{\wt{\Sigma} - \Sigma }_{\mathrm{op}} } \leq \sqrt{\frac{\var{\|X_1 - \mu\|^2}}{n}}\,,
\end{equation}
where $\wt{\Sigma}$ is defined in \eqref{eq:def_sigma_tilde}.
\end{lemma}
\begin{remark}
Using the boundedness of the variables we can upper bound this variance: $\var{\|X_1 -\mu\|^2} \leq 4L^2 \tr \Sigma$.
\end{remark}

\begin{proposition}\label{prop:conc_sigma_tilde_bounded}
\GB{  Let $(X_i)_{1 \leq i \leq n} $ be i.i.d. random vectors in a separable Hilbert space $\cH$,
  with norm bounded by $L$ and covariance $\Sigma$, then for any for $ u \geq 1$}, with probability at least $1-e^{-u}$:
\begin{equation}\label{eq:conc_sigma_tilde_bounded}
\norm[1]{ \wt{\Sigma} - \Sigma }_{\mathrm{op}} \leq 2\sqrt{\frac{\var{\|X_1 - \mu \|^2}}{n}} + L\sqrt{\frac{2\|\Sigma\|_{\mathrm{op}}u}{n} } + \frac{8L^2u}{3n}\,,
\end{equation}
where $\wt{\Sigma}$ is defined in \eqref{eq:def_sigma_tilde}.
\end{proposition}

\begin{proof}
\JB{  We denote in this proof $Z_i := X_i - \mu$ for $1\leq i \leq n$. Let us first remark that if $B_1$ is the unit ball of $\cH$, then:
\begin{equation*}
  \norm[1]{\wt{\Sigma} - \Sigma}_{\mathrm{op}} = \sup_{u,v \in B_1 } \frac{1}{n}\sum_{i=1}^{n} \inner{ v , \paren{ Z_i Z_i^T - \Sigma} u }
  =: \sup_{u,v \in B_1 } \frac{1}{n}\sum_{i=1}^{n} f_{u,v}(X_i)\,.
\end{equation*}
\GB{ Since the variables $X_i$ have norm bounded by $L$,
it can be assumed equivalently that they take their values in $B_L=LB_1$,
and it holds $\sup_{x \in B_L}\sup_{u,v \in B_1} f_{u,v}(x)\leq 8L^2$.
Furthermore, since $(u,v) \mapsto f_{u,v}(x)$ is continuous, and the Hilbert space $\cH$ is separable, the uncountable set $B_1$ can be replaced by a countable dense subset. }
Thus we can apply Theorem \ref{thm:bousquet},} and obtain that with probability at least $1- e^{-x}$:
\begin{align*}
\norm[1]{\wt{\Sigma} - \Sigma}_{\mathrm{op}} \leq 2\e{\|\wt{\Sigma} - \Sigma\|_{\mathrm{op}}} + L\sqrt{\frac{2\|\Sigma\|_{\mathrm{op}}x}{n} } + \frac{16L^2x}{3n}\,,
\end{align*}
where we have used for the variance term:
\begin{align*}
  \underset{u,v \in B_1 }{\sup} \e{\inner{ v , \paren{ Z_iZ_i^T - \Sigma} u }^2} &\leq  \underset{u,v \in B_1 }{\sup} \e{\inner{ v ,  Z_i}^2\inner{Z_i, u }^2} \\
  &\leq 4n L^2 \|\Sigma\|_{\mathrm{op}}\,.
\end{align*}
We conclude using the upper bound of the expectation from Lemma~\ref{lem:upper_bound_expectation_sigma_tilde}.
\end{proof}

  {\em Proof of Proposition \ref{prop:conc_sigma_hat_bounded}.}
  As in the Gaussian case, we have:
  \begin{align*}
    \abs{  {\norm[1]{\widehat{\Sigma}}^{\frac{1}{2}}_{\mathrm{op}}} - {\norm{\Sigma}^{\frac{1}{2}}_{\mathrm{op}}} } & \leq  \abs{ {\norm[1]{\wt{\Sigma}}_{\mathrm{op}}^{\frac{1}{2}}} - {\norm{\Sigma}^{\frac{1}{2}}_{\mathrm{op}}} } + {\norm[1]{ \widehat{\Sigma} - \wt{\Sigma}}^{\frac{1}{2}}_\mathrm{op}}\,.
  \end{align*}
  From Lemma~\ref{lem:sqrt_inequality} and Proposition~\ref{prop:conc_sigma_tilde_bounded}, we have with probability at least $1-e^{-u}$:
  \begin{align*}
     \abs{ {\norm[1]{\wt{\Sigma}}^{\frac{1}{2}}_{\mathrm{op}}} - {\norm{\Sigma}^{\frac{1}{2}}_{\mathrm{op}}} } \leq   4 L \sqrt{\frac{ \tr \Sigma }{n\|\Sigma\|_{\mathrm{op}}}} + \sqrt{\frac{16}{3} \frac{L^2u}{n}}\,,
  \end{align*}
  where we have used that:
  \begin{equation*}
    \sqrt{\frac{\var{\|Z_1\|^2}}{n}} \leq \frac{2L \sqrt{\tr \Sigma}}{\sqrt{n}}\,.
  \end{equation*}
  Using
  \begin{equation*}
    {\norm[1]{ \widehat{\Sigma} - \wt{\Sigma}}^{\frac{1}{2}}_\mathrm{op}} \leq  \|\mu - \widehat{\mu} \|\,,
  \end{equation*}
   and according to Corollary~\ref{cor:conc_norm_bounded}, we obtain that for $ u \geq 0$, with probability at least $1- 2e^{-u}$:
  \begin{align*}
    \abs{  {\norm[1]{\widehat{\Sigma}}^{\frac{1}{2}}_{\mathrm{op}}} - {\norm{\Sigma}^{\frac{1}{2}}_{\mathrm{op}}} }
    & \leq \paren{4 L \sqrt{\frac{ \tr \Sigma }{n\|\Sigma\|_{\mathrm{op}}}}
     + \sqrt{\frac{16}{3} \frac{L^2u}{n}} } \\
    & \qquad \qquad + \paren{2\sqrt{\frac{ \tr \Sigma}{n}} + \sqrt{\frac{2 \|\Sigma\|_{\mathrm{op}} u}{n} } + \frac{4L u}{3n}} \\
    & \leq 8 L \sqrt{\frac{ \tr \Sigma }{n\|\Sigma\|_{\mathrm{op}}}} + 4L\paren{ \sqrt{\frac{2u}{n}} + \frac{u}{3n} }  \,,
  \end{align*}
  where we have used for the last inequality that $\|\Sigma\|_{\mathrm{op}} \leq 4L^2$. \qed

\JB{
\subsection{Proof of Propositions \ref{prop:conc_sqrt_t_trsigma2} and  \ref{prop:conc_sqrt_t_trsigma2_bounded} }\label{se:tr_sigma2}
From a sample $\XX = (X_i)_{1\leq i \leq n}$ of i.i.d. random vectors, we want to estimate $\tr \Sigma^2$ where $\Sigma$ is their common covariance matrix. The statistic $\widehat{T}$ defined in~\eqref{eq:def_t_trsigma2}
is an unbiased estimator of $\tr \Sigma^2$.
This statistic is also invariant by translation.

constant ($\nabla_\mu \tau = 0$).

If we denote $\mathfrak{S}_n$ the set of permutations of $\{1,\ldots, n\}$, $\widehat{T}$ can be rewritten as:
\begin{equation}\label{eq:T_sym}
  \widehat{T } = \frac{1}{n!} \sum_{\sigma \in \mathfrak{S}_n} \frac{1}{\lfloor {n}/{4} \rfloor} \sum_{i=1}^{\lfloor {n}/{4} \rfloor}
    \frac{1}{4} \inner{ X_{\sigma(4i)} - X_{\sigma(4i-2)} , X_{\sigma(4i-1)}- X_{\sigma(4i-3)}}^2;
\end{equation}
namely by symmetry, all the 4-tuples appear the same number of times in the right-hand side,
so we just need to divide by the number of terms to obtain the identity \eqref{eq:T_sym}. We will use this decomposition to obtain a concentration of the statistic $\widehat{T}$ for the Gaussian case and the bounded case,
since the inner sum for each fixed permutation is a sum of $\lfloor {n}/{4} \rfloor$ i.i.d. terms.

\subsubsection{Gaussian \JB{setting}.}

Because the statistic is invariant by translation we can assume without loss of generality that $\mu=0$. To
obtain a deviation inequality for $\widehat{T}^{\nicefrac{1}{2}}$, we will first find a concentration inequality for $\widehat{T}$ and then use Lemma~\ref{lem:sqrt_inequality}.
We obtain concentration via control of moments of $\wh{T}$, so we first need some upper bounds on Gaussian moments. The following lemma is proved in Section~\ref{se:supp}.
\begin{lemma}\label{lem:moment_T}
  Let $Z_i := \inner{X_i^1-X_i^3,X_i^2-X_i^4}^2/4$, 
  where $X_i^{j} $ for $i=1,\ldots,m$ and $1\leq j \leq 4$ are i.i.d. Gaussian random vectors $\cN(0, \Sigma)$. Then for all $q \in \mbn$:
\begin{equation}\label{eq:moment_T}
  \e{ \paren{ \frac{1}{m} \sum_{i=1}^{m} Z_i - \tr \Sigma^2 }^{2q} } \leq
  \GB{\paren{4\sqrt{2}\phi q^2 \frac{\tr \Sigma^2}{\sqrt{m}}}^{2q},}
\end{equation}
where $\phi =(1+ \sqrt{5})/2$ is the golden ratio.
\end{lemma}
We deduce from this lemma a concentration inequality for $\widehat{T}$.
\begin{proposition}\label{prop:conc_t_trsigma2}
Let \GB{$(X_i)_{1 \leq i \leq n}$, $n\geq 4$ be i.i.d. random vectors with distribution $\cN(\mu ,\Sigma)$.} Then for all $u \geq 0 $:
\begin{equation}\label{eq:conc_t_trsigma2}
  \prob{ \abs{\widehat{T} - \tr \Sigma^2 } \geq \GB{30}\frac{u^2\tr \Sigma^2}{\sqrt{n}} }   \leq e^4e^{-u}\,,
\end{equation}
where $\widehat{T}$ is defined in \eqref{eq:def_t_trsigma2}.
\end{proposition}

\begin{proof}
Using Lemma~\ref{lem:moment_T}, \eqref{eq:T_sym} and the convexity of the function $x \mapsto x^{2q}$, we can upper bound the moments of $\widehat{T}$:
\begin{equation} \label{eq:momthat}
  \e{\paren[1]{\widehat{T} - \tr \Sigma^2}^{2q}} \leq
  \GB{\paren{4\sqrt{2}\phi q^2 \frac{\tr \Sigma^2}{\sqrt{\lfloor n/4 \rfloor}}}^{2q}}.
\end{equation}
Let $t \geq 0$ and $q \in \mbn$, then \GB{by Markov's inequality}
  \begin{equation} \label{eq:devthat}
    \prob{ \abs[1]{ \widehat{T} - \tr \Sigma^2 } \geq t } \leq  t^{-2q} \e{\paren[1]{\widehat{T} - \tr \Sigma^2}^{2q}} \,.
  \end{equation}
  Let us choose $q$ as:
\[
  q = \left\lfloor \GB{\frac{e^{-1}}{2\sqrt{\phi}2^{\frac{1}{4}}}
    t^{\frac{1}{2}} \paren{\frac{ \tr \Sigma^2}{\sqrt{\lfloor n/4 \rfloor}}}^{-\frac{1}{2}}}
    \right\rfloor\,,
  \]
  so that~\eqref{eq:momthat}, \eqref{eq:devthat} entail
  \begin{equation*}
    \prob{ \abs[1]{ \widehat{T} - \tr \Sigma^2 } \geq t } \leq e^{-4q} \,.
  \end{equation*}
  Let us now take
\[
  t = \frac{e^2\GB{\sqrt{2}\phi}}{\GB{4}} u^2 \frac{ \tr \Sigma^2}{\sqrt{\lfloor n/4 \rfloor}}
  \leq \GB{30}\frac{u^2\tr \Sigma^2}{\sqrt{n}}\,,
  \]
\GB{  where we have used $\lfloor {n}/{4} \rfloor \geq {n}/{7}$ for $n\geq 4$};
  we obtain  that for all $u\geq 0$:
  \[
    \prob{ \abs[1]{\widehat{T} - \tr \Sigma^2 } \geq \GB{30}\frac{u^2\tr \Sigma^2}{\sqrt{n}} } \leq e^4e^{-u}\,.
  \]\qed
\end{proof}
\GB{Proposition~\ref{prop:conc_sqrt_t_trsigma2} directly follows from
Proposition~\ref{prop:conc_t_trsigma2} and Lemma~\ref{lem:sqrt_inequality}.}

\subsubsection{Bounded \JB{setting}.}

As in the Gaussian case, we first obtain a concentration inequality for $\widehat{T}$ and then using Lemma \ref{lem:sqrt_inequality}, we obtain one for ${\widehat{T}^{\nicefrac{1}{2}}}$. We will need the following classical Bernstein's inequality  (see for instance \cite{Ver19}, Exercise~2.8.5 for the version below) which gives an upper bound on the Laplace transform of the
sum of bounded random variables.

\begin{lemma}[Bernstein's inequality]\label{lem:lap_trans_bounded_sum}
Let $(X_i)_{1\leq i \leq m}$ be i.i.d. real centered random variables bounded by $B$ such that
\begin{equation*}
  \e{X_1^2} \leq \sigma^2\,.
\end{equation*}
Then for all $t <3/B$:
\begin{equation*}
  \log \paren{\e[1]{e^{t \sum X_i } } }\leq \frac{1}{2}\frac{m\sigma^2 t^2}{1-Bt/3}\,.
\end{equation*}
\end{lemma}
\GB{Via Bernstein's inequality we obtain the following result.}
\begin{proposition}\label{prop:conc_t_trsigma2_bounded}
\GB{Let $(X_i)_{1\ \leq i \leq n}$, $n\geq 4$ be i.i.d. Hilbert-valued random variables with norm bounded by $L$ and covariance $\Sigma$, and $\wh{T}$ defined by~\eqref{eq:def_t_trsigma2}}. Then for all $t \geq 0 $:
\begin{equation}\label{eq:conc_t_trsigma2_bounded}
      \prob{ \abs{\widehat{T} - \tr \Sigma^2 } \geq \GB{8}L^2 \sqrt{\frac{\tr \Sigma^2 t}{n }} + \frac{\GB{10}L^4t}{ n } } \leq \GB{2} e^{-t}\,.
\end{equation}
where $\widehat{T}$ is defined in \eqref{eq:def_t_trsigma2}.
\end{proposition}

\begin{proof}
\GB{  Let $X$, $X'$, $Y$, $Y'$ be i.i.d. Hilbert-valued random vectors of expectation $\mu$, covariance $\Sigma$ and with norm
  bounded by $L$,
  and $Z:=\inner{X-Y,X'-Y'}^2/4$.
  Then it holds $0 \leq Z \leq 4L^4$, $\e{Z} = \tr\Sigma^2$ and
\begin{align*}
  \abs{
  Z - \e{Z} } & \leq 4 L^4;\\
  \var{Z} &
  \leq 4L^4 \e{Z} = 4L^4\tr\Sigma^2.
\end{align*}}
Now using the convexity of the exponential function, \eqref{eq:T_sym} and then Lemma~\ref{lem:lap_trans_bounded_sum}, we can upper bound the Laplace transform of $\widehat{T}$ as follows:
\begin{align*}
  \log\paren{\e[1]{e^{t\widehat{T}}}} &\leq \frac{1}{2\lfloor n/4 \rfloor}\frac{ 4L^4\tr \Sigma^2  t^2}{1-
                                     \GB{4}L^4t/(3\lfloor n/4 \rfloor)}\,,
\end{align*}
for all $t$ such that the right-hand sise is well defined, i.e. the denominator is strictly positive.
Now using Lemma~\ref{lem:birge}, \GB{and $\lfloor n/4 \rfloor \geq n/7$ for $n\geq 4$,} for all $t \geq 0$ it holds
  \begin{equation}\label{eq:probdevThat}
    \prob{ \abs[2]{\widehat{T} - \tr \Sigma^2 } \geq \GB{8}L^2 \sqrt{\frac{\tr \Sigma^2 t}{n }} + \frac{\GB{10}L^4t}{ n } } \leq 2 e^{-t}\,.
  \end{equation}

\end{proof}

\paragraph{\bf Proof of Proposition \ref{prop:conc_sqrt_t_trsigma2_bounded}.}
\GB{Assuming the event entering into~\eqref{eq:probdevThat} holds,}
we will use the inequalities of Lemma \ref{lem:sqrt_inequality}:
  \begin{align*}
    \sqrt{\widehat{T}} - \sqrt{\tr \Sigma^2} &\leq  \sqrt{\tr \Sigma^2 + \GB{8}L^2 \sqrt{\frac{\tr \Sigma^2 t}{n }}} - \sqrt{\tr \Sigma^2} + \sqrt{ \frac{\GB{10}L^4t}{ n }} \\
    & \leq \GB{4}L^2 \sqrt{\frac{ t}{n }} + L^2 \sqrt{ \frac{\GB{10}t}{n}} \leq \GB{8}L^2 \sqrt{\frac{t}{n }} \,.
  \end{align*}
  For the other side, we proceed analogously:
  \begin{align*}
    \sqrt{\widehat{T}} - \sqrt{\tr \Sigma^2}&  \geq \sqrt{\paren{\tr \Sigma^2 - \GB{8}L^2 \sqrt{\frac{\tr \Sigma^2 t}{n }}}_+} - \sqrt{\tr \Sigma^2}
                                              - \sqrt{ \frac{\GB{10}L^4t}{ n }} &\\
    & \geq - \GB{8}L^2 \sqrt{\frac{ t}{n }} - L^2 \sqrt{ \frac{\GB{10}t}{n}} \geq -\GB{12}L^2 \sqrt{\frac{t}{n }}\,.
  \end{align*}\qed

}

\subsection{Additional proofs}
\label{se:supp}

 \paragraph{\bf Proof of Lemma \ref{lem:sqrt_inequality}.}
 This Lemma completes the Lemma 6.1.3 of \citet{BlaCarGut18}. This is its complete proof.

 Let $a$ in $\mbr_+$, it is well known that for  $b \geq -a^2$:
 \begin{equation*}
  a - \sqrt{|b|}  \leq \sqrt{a^2 + b} \leq a+ \sqrt{|b|}\,.
 \end{equation*}
On the other hand, suppose that $b \geq 0$, the Taylor expansion of the function $b \mapsto \sqrt{a^2+ b} - a$ gives that there exists $c \in (0, b)$ such that:
\begin{equation*}
    \sqrt{a^2 + b} - a = \frac{b}{2 \sqrt{a^2+ c}} \leq \frac{b}{2a}\,.
\end{equation*}
Suppose now that $0 \geq b \geq -a^2$, then
\begin{equation*}
  \sqrt{a^2+b } \geq a + \frac{b}{a} \Leftrightarrow b \geq 2b + \frac{b^2}{a^2} \Leftrightarrow b \geq -a^2\,.
\end{equation*}
The equation \eqref{eq:sqrt_inequality} is still true when $b < -a^2$ because then:
\begin{equation*}
  -a \geq - \sqrt{|b|} \geq - \frac{|b|}{a}\,.
\end{equation*}\qed

\paragraph{\bf Proof of Proposition~\ref{prop:lapl_trans_gauss_vector_inner}.}
Let $g$ be a standard Gaussian random vector in $\mbr^d$, and $U^TDU$ be the singular value decomposition of the matrix $S^{\nicefrac{1}{2}}\Sigma S^{\nicefrac{1}{2}}$ where $D = \mathrm{diag}(\lambda_,\ldots,\lambda_d)$. Then we have the following equalities in distribution
\begin{equation*}
    Y^T\Sigma Y \overset{\mathrm{dist}}{\sim} g^T S^{\frac{1}{2}} \Sigma S^{\frac{1}{2}} g \overset{\mathrm{dist}}{\sim} g^T U^T D U g \overset{\mathrm{dist}}{\sim} g^T D g\,.
\end{equation*}
The last equality is a consequence of the invariance by rotation of Gaussian vectors. Then for $ t < 1/\sqrt{\|\Sigma\|_{\mathrm{op}}\|S\|_{\mathrm{op}}}$:
\begin{align*}
  \e{e^{t\inner{X,Y}}} &= \e{e^{\frac{t^2 \|\Sigma^{\frac{1}{2}}Y \|^2}{2}}} = \e{ e^{\frac{t^2g^T D g}{2}}} = \e{ \exp\paren{ \frac{t^2}{2} \sum_{i=1}^{d} \lambda_i g_i^2 }}\,.
\end{align*}
Using the independence of the coordinates and that $- \log ( 1- x) \leq \frac{x}{1-x} \leq \frac{x}{1-\sqrt{x}}$ for $x <1$
(the first inequality can easily be checked by termwise power series comparison), we obtain:
\begin{multline*}
  \log \paren{ \e{e^{t\inner{X,Y}}} } = \sum_{i=1}^{n} - \frac{1}{2} \log\paren{ 1  - t \sqrt{\lambda_i} } \\
  \leq \sum_{i=1}^{n} \frac{1}{2} \frac{t^2\lambda_i}{1 - t^2  \lambda_i} \leq \frac{1}{2} \frac{t^2 \tr (S^{\frac{1}{2}} \Sigma S^{\frac{1}{2}})}{1 - t \|S^{\frac{1}{2}} \Sigma S^{\frac{1}{2}}\|_{\mathrm{op}}^{\frac{1}{2}}}\,.
\end{multline*}
We conclude using that $\tr (S^{\frac{1}{2}} \Sigma S^{\frac{1}{2}}) = \tr (\Sigma S)$ and that $ \|S^{\frac{1}{2}} \Sigma S^{\frac{1}{2}}\|_{\mathrm{op}} \leq \|S\|_{\mathrm{op}} \|\Sigma\|_{\mathrm{op}}$. \qed

\paragraph{\bf Proof of Corollary~\ref{cor:conc_gauss_classic}.}
We use the representation $X \stackrel{\mathrm{dist}}{\sim} (\Sigma^{\frac{1}{2}}g+ \mu )$, where $g$ is a standard Gaussian random variable.
We then have
\begin{equation*}
 \|X \|_d \stackrel{\mathrm{dist}}{\sim} \| \Sigma^{\frac{1}{2}}g + \mu \|_d = f(g)\,,
\end{equation*}
where for $y \in \mbr^d$:
  \begin{equation*}
    f(y) =  \norm[1]{ \Sigma^{\frac{1}{2}}y + \mu }_d\,.
  \end{equation*}
  This function $f$ is Lipschitz with constant $\|\Sigma^{\frac{1}{2}}\|_{\mathrm{op}}$. We conclude using Theorem~\ref{thm:Gaussian_lipschitz_conc} and Jensen's inequality:
  \begin{equation*}
    \e[1]{\|X\|_d} \leq \sqrt{\|\mu\|_d^2+ \tr \Sigma}\,.
  \end{equation*}\qed

\paragraph{\bf Proof of Corollary~\ref{cor:conc_norm_bounded}.}
We apply Theorem~\ref{thm:bousquet}, with $\varepsilon = 1$ and the set of functions $ \cF = \{ f_u \}_{\|u\|_\cH =1}$ where $f_u: x \in \cH \mapsto \inner{x,u }_\cH$ for $u\in \cH$. We can find a countable subset of the unit sphere because $\cH$ is separable. Then
\begin{equation*}
  Z = \sup_{ \|u\|_\cH =1 } \sum_{i=1}^{n} \inner{X_i- \mu, u}_\cH = \norm[3]{\sum_{i=1}^{n} X_i - \mu }_\cH\,.
\end{equation*}
We conclude using that for all $u$ in the unit sphere of $\cH$,  $\var{\inner{X_i - \mu, u}_\cH} \leq \|\Sigma\|_{\mathrm{op}}$ and $\abs{ \inner{X_i- \mu, u}_\cH} \leq 2L$ a.s. We use Jensen's inequality to upper bound the expectation: $\e{Z} \leq (n \tr \Sigma)^{\frac{1}{2}}$.
\qed

\paragraph{\bf Proof of  Lemma \ref{lem:upper_bound_expectation_sigma_tilde}.}
  We upper bound the operator norm with the Frobenius norm. We denote in this proof $Z_i := X_i - \mu$. It holds:
\begin{multline*}
  \e{ \norm[1]{\Sigma - \wt{\Sigma} }_{\mathrm{op}}}\\
  \begin{aligned}
  & \leq \e{ \sqrt{\tr \paren[1]{\Sigma - \wt{\Sigma}}^2}} \\
  & \leq \paren{ \e[3]{ \tr \paren[3]{ \frac{1}{n^2} \paren[2]{\sum_{i} (Z_iZ_i^T)^2 
    +\sum_{i\neq j} Z_iZ_i^TZ_jZ_j^T}- \wt{\Sigma} \Sigma - \Sigma \wt{\Sigma}+ \Sigma^2 }}}^{\frac{1}{2}} \\
  & = \paren{ \frac{\e{\|Z\|^4}}{n} - \frac{\tr \Sigma^2}{n} }^{\frac{1}{2}} \\
  & = \sqrt{\frac{\var{\|Z\|^2}}{n}}\leq \frac{2L \sqrt{\tr \Sigma}}{\sqrt{n}}\,.
\end{aligned}
\end{multline*}
\qed

\JB{
 \paragraph{\bf Proof Lemma \ref{lem:moment_T}.} First let us remark that if $X$ and $X'$ are independent $\cN(0,\Sigma)$ Gaussian vectors, then
\begin{equation*}
  \inner{ X , X' } \overset{\mathrm{dist}}{\sim} \sum_{i=1}^{d} \lambda_i g_i g_i'\,,
\end{equation*}
where $g_i$ and $g_i'$ are independent standard Gaussian random variables and the $\lambda_i$s are the eigenvalues of $\Sigma$. Then for $q \in \mbn$, \GB{recalling $\e[1]{g_i^{2q}} = (2q!)/(2^q q!)$,}
\begin{align*}
  \e{\inner{
   X, X' }^{2q} } &= \sum_{p_1+\ldots+p_d = q } \binom{2q}{2p_1,\ldots,2p_d} \prod_{i=1}^{d} (\lambda_i)^{2p_i}
 \paren{ \frac{(2p_i)!}{2^{p_i} p_i!}}^2 \\
 & \leq (2q)! \sum_{p_1+\ldots+p_d = q}   \prod_{i=1}^{d} (\lambda_i^2)^{p_i}  \\
 & \leq (2q)! ( \tr \Sigma^2)^q,
 \end{align*}
where we have used $(2p)! \leq 2^{2p} p!^2$.
Using this bound,
we upper bound the moments of the $Z_i$s:
\begin{equation*}
  \abs[1]{\e{Z^q_i}} =  2^{-2q} \e{\inner{X_i^1-X_i^3,X_i^2-X_i^4 }^{2q}}
  \leq (2q)! (\tr \Sigma^2)^q\,.
\end{equation*}
We now upper bound the moments of $Z_i -\tr \Sigma$. Let $Z'_i$ be an independent copy of $Z_i$, then
since $\e{Z'_i} = \tr \Sigma^2$, by Jensen's inequality
\begin{align*}
  \e{ \paren{ Z_i - \tr \Sigma^2 }^{2q}} & \leq  \e{ \paren{ Z_i - Z'_i}^{2q}} \leq 2^{2q} \e{Z_i^{2q}} \leq (4q)!(
                                           \GB{2} \tr \Sigma^2)^{2q}\,.
  \end{align*}
  For the odd moments we use that the function $(\cdot)^{2q+1}$ is increasing:
   \begin{align*}
     -(\tr \Sigma^2)^{2q+1} \leq  \e{ \paren{ Z_i - \tr \Sigma^2 }^{2q+1}}\leq  \e{  Z_i^{2q+1} }\leq (4q+2)! (\tr \Sigma^2)^{2q+1}\,,
   \end{align*}
   so for all $q \geq 0 $:
   \begin{equation}
     \left| \e{ \paren{Z_i - \tr \Sigma^2 }^{q}} \right| \leq  (2q)! (\GB{2}\tr \Sigma^2)^q\,.
   \end{equation}
   It remains to upper bound the moments of the sum:
    \begin{multline*}
      \e{ \paren{\frac{1}{m} \sum_{i=1}^{m} Z_i^2 - \tr \Sigma^2 }^{2q} }\\
      \begin{aligned}
        & = \frac{1}{m^{2q}}\sum_{\substack{ p_1+\ldots+p_m = 2q \\ p_i \neq 1 }} \binom{2q}{p_1,\ldots,p_m} \prod_{i=1}^{m} \e{ \paren{Z_i - \tr \Sigma^2 }^{p_i} } \\
       & \leq \frac{1}{m^{2q}}\sum_{\substack{ p_1+\ldots+p_m = 2q \\ p_i \neq 1 }} \frac{(2q)!}{p_1!\ldots p_m!}  \prod_{i=1}^{m}  (2p_i)! (\GB{2}\tr \Sigma^2)^{p_i} \\
       & \leq (2q)! \paren{\frac{\GB{2}\tr \Sigma^2}{m}}^{2q} (2q)^{2q} \sum_{\substack{ p_1+\ldots+p_m = 2q \\ p_i \neq 1 }} 1.
     \end{aligned}
     \end{multline*}
    \JB{Let us count the number of terms in this last sum. Consider first that we have $k$ non-null terms $(p_{i_1},\ldots , p_{i_k})$. Their sum is equal to $2q$ but because these terms are strictly greater than $1$, we also have that $(p_{i_1}-2)+\ldots + (p_{i_k}-2) = 2q-2k$, where all terms of this sum are nonnegative. The number of $k$-partitions of $2q-2k$ is $\binom{(2q-2k)+(k-1)}{k-1}= \binom{2q-k-1}{k-1} $  and then the number of terms in the sum is equal to:}
    \begin{align*}
      \sum_{k=0}^{m} \binom{m}{k} \binom{2q-k-1}{k-1} &= \sum_{k=0}^{m \wedge q} \binom{m}{k} \binom{2q-k-1}{k-1} \\
       &\leq m^q  \sum_{k=0}^{ q} \binom{2q-k-1}{k-1} = m^q F(2q-1) \leq m^q \phi^{2q}\,,
    \end{align*}
where $F(\cdot)$ is the Fibonacci sequence and $\phi = (1+ \sqrt{5})/2$ is the golden ratio. So using that $(2q)! \leq (2q)^qq^q$ we obtain that
    \begin{equation}
        \e{ \paren{ \GB{\frac{1}{m}} \sum_{i=1}^{m} Z_i - \tr \Sigma^2 }^{2q} } \leq  (\GB{2} \phi^2)^q \paren{\frac{ \tr \Sigma^2}{\sqrt{m}}}^{2q} (2q)^{4q}.
    \end{equation}\qed
    }

  {\bf Acknowledgements.}
  GB acknowledges support from: Deutsche Forschungsgemeinschaft (DFG) - SFB1294/1 - 318763901;
  Agence Nationale de la Recherche (ANR), ANR-19-CHIA-0021-01 ``BiSCottE''; the
  Franco-German University (UFA) through the binational Doktorandenkolleg CDFA 01-18.

\printbibliography

\end{document}